\documentclass[10pt]{article} 
\usepackage[preprint]{neurips_2024}

\usepackage{graphicx}
\usepackage{amsfonts}
\usepackage{amssymb}
\usepackage{url}
\usepackage{fancyhdr}
\usepackage{indentfirst}
\usepackage{enumerate}
\usepackage{amsthm}
\usepackage{amsmath}
\usepackage{amssymb}
\usepackage{mathtools}
\usepackage{dsfont}
\usepackage{bbm}

\providecommand{\U}[1]{\protect\rule{.1in}{.1in}}
\newcommand{\diff}{\mathrm{d}}

\newcommand{\E}{\mathbb{E}}
\newcommand{\EE}{\mathbb{E}}

\newcommand{\R}{\mathbb{R}}

\newcommand{\mc}{\mathcal}

\newcommand{\PP}{\mathbb{P}}
\newcommand{\QQ}{\mathbb{Q}}

\renewcommand{\ge}{\geqslant}

\renewcommand{\geq}{\geqslant}
\renewcommand{\leq}{\leqslant}
\renewcommand{\epsilon}{\varepsilon}

\theoremstyle{definition}
\newtheorem{theorem}{Theorem}

\newtheorem{lemma}[theorem]{Lemma}
\newtheorem{proposition}[theorem]{Proposition}

\theoremstyle{definition}
\newtheorem{definition}{Definition}[section]
\newtheorem{example}{Example}[section]

\theoremstyle{remark}
\newtheorem{remark}{Remark}[section]
\theoremstyle{definition}

\renewcommand{\cite}{\citet}
\renewcommand{\cdots}{\dots}

\newcommand{\st}{\mathrm{s.t.}}

\usepackage[ruled,linesnumbered]{algorithm2e}
\usepackage{algpseudocode}
\usepackage[utf8]{inputenc} 
\usepackage[T1]{fontenc}    
\usepackage{hyperref}       
\hypersetup{colorlinks, breaklinks,
linkcolor={red!80!black},
citecolor={blue!70!black},
urlcolor={blue!70!black}}
\usepackage{url}            
\usepackage{booktabs}       
\usepackage{amsfonts}       
\usepackage{nicefrac}       
\usepackage{microtype}      
\usepackage{xcolor}         
\usepackage{tabularx}

\usepackage{amsmath}
\usepackage{amssymb}
\usepackage{mathtools}
\usepackage{amsthm}
\usepackage{enumitem} 
\usepackage{graphicx}
\usepackage{subfigure}
\usepackage{hyperref}
\usepackage{multirow}

\usepackage{physics}
\usepackage{amsmath}
\usepackage{tikz}
\usepackage{mathdots}
\usepackage{yhmath}
\usepackage{cancel}
\usepackage{color}
\usepackage{siunitx}
\usepackage{array}
\usepackage{multirow}
\usepackage{amssymb}
\usepackage{gensymb}
\usepackage{tabularx}
\usepackage{extarrows}
\usepackage{booktabs}
\usepackage{wrapfig}
\usetikzlibrary{fadings}
\usetikzlibrary{patterns}
\usetikzlibrary{shadows.blur}
\usetikzlibrary{shapes}

\title{Automatic Outlier Rectification via Optimal Transport}

\author{
Jose Blanchet \\
Department of Management Science \& Engineering\\
Stanford University\\
\texttt{jose.blanchet@stanford.edu}
\AND
Jiajin Li \\
Department of Management Science \& Engineering\\
Stanford University\\
\texttt{jiajinli@stanford.edu}
\AND
Markus Pelger \\
Department of Management Science \& Engineering\\
Stanford University\\
\texttt{mpelger@stanford.edu}
\AND
Greg Zanotti \\
Department of Management Science \& Engineering\\
Stanford University\\
\texttt{gzanotti@stanford.edu}
}

\newcommand{\dsg}{\nabla\hat S}

\begin{document}

\maketitle

\begin{abstract}
In this paper, we propose a novel conceptual framework to detect outliers using optimal transport with a concave cost function. Conventional outlier detection approaches typically use a two-stage procedure: first, outliers are detected and removed, and then estimation is performed on the cleaned data. However, this approach does not inform outlier removal with the estimation task, leaving room for improvement. 
To address this limitation, we propose an automatic outlier rectification mechanism that integrates rectification and estimation within a joint optimization framework. 
We take the first step to utilize the optimal transport distance with a concave cost function to construct a rectification set in the space of probability distributions. Then, we select the best distribution within the rectification set to perform the estimation task. Notably, the concave cost function we introduced in this paper is the key to making our estimator effectively identify the outlier during the optimization process. 
We demonstrate the effectiveness of our approach over conventional approaches in simulations and empirical analyses for mean estimation, least absolute regression, and the fitting of option implied volatility surfaces.

\end{abstract}

\vspace{-8pt}
\section{Introduction}
\vspace{-8pt}
Outlier removal has a long tradition in statistics, both in theory and practice. This is because it is common to have (for example, due to collection errors) data contamination or corruption. Directly applying a learning algorithm to corrupted data can, naturally, lead to undesirable out-of-sample performance. Our goal in this paper is to provide a single-step optimization mechanism based on optimal transport for automatically removing outliers. 

The challenge of outlier removal has been documented for centuries: early work is introduced in e.g. \citep{gergonne1821dissertation} and \citep{peirce1852criterion}. Yet the outlier removal problem continues to interest practitioners and researchers alike due to the danger of distorted model estimation. A natural family of approaches followed in the literature takes the form of ``two-stage'' methods, which involve an outlier removal step followed by an estimation step. Methods within this family range from rules of thumb, such as removing outliers beyond a particular threshold based on a robust measure of scale like the interquartile range \citep{tukey1977exploratory}, to various model-based or parametric distributional assumption-based tests \citep{thompson1985note}.
While the two-stage approach can be useful by separating the estimation task from the outlier removal objective, it is not without potential pitfalls. For example, outlier detection which relies on a specific fitted model may tend to overfit to a particular type of outlier, which can mask the effect of other types \citep{rousseeu1987robust}. Conversely, the outlier detection step may be overly conservative if it is not informed by the downstream estimation task; this can lead to a significant reduction in the number of observations available for model estimation in the next step, resulting in a loss of statistical efficiency \citep{he1992reweighted}.

Robust statistics, pioneered by \citep{box1953non}, \citep{tukey1960survey}, \citep{huber1964robust} and others such as \citep{hampel1968contributions, hampel1971general} offers alternative approaches for obtaining statistical estimators in the presence of outliers without removing them, particularly in parametric estimation problems such as linear regression. Beyond these, one closely related approach is the minimum distance functionals-based estimator, introduced by \citep{parr1980minimum,millar1981robust,donoho1988automatic,donoho1988pathologies,park1995robust,zhu2022generalized,jaenada2022robust}. 
This method involves projecting the corrupted distribution onto a family of distributions 
using a distribution metric and selecting the optimal estimator for the resulting distribution. However, this projection mechanism is not informed by the estimation task (such as fitting a linear regression or neural network). Thus, without additional information about the contamination model or estimation task, it can be challenging to choose an appropriate family of distributions for projection, which may lead to limitations similar to those outlined above in practice.

In this paper, we propose a novel approach to integrate both outlier detection and estimation in a joint optimization framework. A key observation is that statisticians aim to ``clean'' data \textbf{before} decisions are made; thus, ideal robust statistical estimators tend to be optimistic. To address this, we introduce the \textit{rectification set}: a ball centered around the empirical  (contaminated) distribution and defined by the optimal transport distance  (see \citep{villani2009optimal,peyre2019computational}) in probability space. This rectification set aims to exclude potential outliers and capture the true underlying distribution, allowing us to minimize the expectation over the ``best-case'' scenarios, leading to a min-min problem. The study in \cite{jiang2024distributionally} also explores connections with min-min-type problems, but concentrates on the construction of  artificially constructed rectification sets to cover certain existing estimators from the robust statistics literature. However, our primary focus is on introducing a novel rectification set, which is based on the optimal transport approach with a concave cost function. 

To automatically detect outliers during the estimation process, one of our main contributions is the use of a concave cost function for the optimal transport distance. This function encourages what we refer to as ``long haul'' transportation, in which the optimal transport plan moves only a portion of the data to a distant location, while leaving other parts unchanged.
This strategic approach effectively repositions identified outliers closer to positions that better align with the clean data. Our novel formulation then involves minimizing the expected loss under the optimally rectified empirical distribution. This rectification is executed through the application of an optimal transport distance with a concave cost function, thereby correcting outliers to enhance performance within a fixed estimation task. More importantly, our method distinguishes itself from distributionally robust optimization (DRO)~\citep{ben2013robust,bayraksan2015data,wiesemann2014distributionally,delage2010distributionally,gao2022distributionally,blanchet2017distributionally,shafieezadeh2015distributionally,shafieezadeh2019regularization,sinha2018certifying,kuhn2019wasserstein} due to the distinctive correction mechanisms initiated by the robust formulation employing a min-min strategy.

As we discuss in Section \ref{sec:formulation} below, the timing of error generating differs crucially between DRO and robust statistics. 
DRO approach employs a min-max game strategy to control the worst-case loss over potential post-decision distributional shifts.  In contrast, the robust estimator acts after the pre-decision distributional contamination materializes. Thus the approach of robust statistics can be motivated as being closer to a max-min game against nature. As a consequence, in robust statistics, the adversary moves first, and therefore the statistician can be more optimistic that they can rectify the contamination applied by  the nature thus motivating the min-min strategy suggested above. However, it is also worth mentioning that \citep{kang2023distributionally, nietert2024outlier} are also trying to formulate the outlier-robust problem in a min-max form, incorporating additional information about the contamination model to further shrink the ambiguity set. 

We summarize our main contributions as follows:

\vspace{-4pt}
\textbf{\textrm{(i)} Novel statistically robust estimator.} We propose a new statistically robust estimator that incorporates a novel rectification set constructed using the optimal transport distance. By employing a concave cost function within the optimal transport distance, our estimator enables automatic outlier rectification. We prove that the optimal rectified distribution can be found via a finite convex program, and show that we can determine the optimal rectification set uncertainty budget via cross-validation.

\vspace{-4pt}
\textbf{\textrm{(ii)} Connection to adaptive quantile regression.} For mean estimation and least absolute regression, we demonstrate that our robust estimator is equivalent to an adaptive quantile (regression) estimator, with the quantile controlled by the budget parameter $\delta$. Furthermore, we prove that the optimal rectified distribution exhibits a long haul structure that facilitates outlier detection.

\vspace{-4pt}
\textbf{\textrm{(iii)} Effectiveness in estimating the option implied volatility surface.} We evaluate our estimator on various tasks, including mean estimation, least absolute regression, and option implied volatility surface estimation. Our experimental results demonstrate that our estimator produces surfaces that are 30.4\% smoother compared to baseline estimators, indicating success in outlier removal, as surfaces should be smooth based on structural financial properties. Additionally, it achieves an average reduction of 6.3\% in mean average percent error (MAPE) across all estimated surfaces, providing empirical evidence of the effectiveness of our rectifying optimal transporter.

\vspace{-10pt}
\section{DRO and Robust Statistics as Adversarial Games}
\label{sec:formulation}
\vspace{-8pt}

In this section, we summarize conceptually how robust statistics is different from DRO (for more details, we refer the reader to e.g. \cite{blanchet2024distributionally}). To lay a solid mathematical foundation, we begin by investigating a generic stochastic optimization problem. Here, we assume that $Z$ is a random vector in a space $\mc Z \subseteq\R^d$ that follows the distribution $\PP_{\star}$. The set of feasible model parameters is denoted $\Theta$ (assumed to be finite-dimensional to simplify). Given a realization $z$ and a model parameter $\theta \in \Theta$ the corresponding loss is $\ell(\theta, z)$. A standard expected loss minimization decision rule is obtained by solving
\vspace{-5pt}
\begin{align}\label{eq:ERM}
    \min_{\theta \in \Theta} ~\EE_{\PP_\star}[\ell(\theta, \xi)] = \int_{\Xi} \ell(\theta, \xi)~\diff \PP_\star (\xi).
    \vspace{-5pt}
\end{align}
Since $\PP_{\star}$ is generally unknown, to approximate the objective function in \eqref{eq:ERM}, we often gather $n$ independent and identically distributed (i.i.d.) samples $\{z_i\}_{i=1}^n$ from the unknown data-generating distribution $\mathbb{P}_{\star}$ and consider the empirical risk minimization counterpart,  
\vspace{-5pt}
\begin{align}\label{eq:erm}
    \min_{\theta \in \Theta} ~\EE_{\hat \PP_n}[\ell(\theta, Z)] = \frac{1}{n} \sum_{i=1}^{n} \ell(\theta, z_i),
    \vspace{-5pt}
\end{align}
where $\hat \PP_n$ denotes the empirical measure $\frac{1}{n}\sum_{i=1}^{n}\delta_{z_i}$ and $\delta_{z}$ is the Dirac measure centered at $z$. These problems are solved within the context of the general data-driven decision-making cycle below:

\vspace{-5pt}
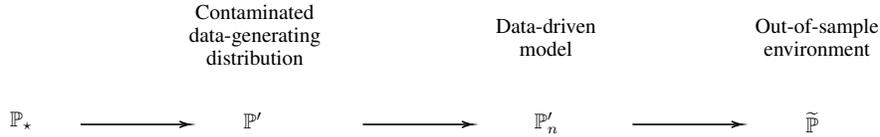
\begin{figure}[ht]
\centering
\scalebox{0.8}{
\tikzset{every picture/.style={line width=0.75pt}} 

\begin{tikzpicture}[x=0.75pt,y=0.75pt,yscale=-1,xscale=1]

\draw    (272,130.3) -- (339,130.3) ;
\draw [shift={(341,130.3)}, rotate = 180] [color={rgb, 255:red, 0; green, 0; blue, 0 }  ][line width=0.75]    (6.56,-1.97) .. controls (4.17,-0.84) and (1.99,-0.18) .. (0,0) .. controls (1.99,0.18) and (4.17,0.84) .. (6.56,1.97)   ;
\draw    (94,130.3) -- (161,130.3) ;
\draw [shift={(163,130.3)}, rotate = 180] [color={rgb, 255:red, 0; green, 0; blue, 0 }  ][line width=0.75]    (6.56,-1.97) .. controls (4.17,-0.84) and (1.99,-0.18) .. (0,0) .. controls (1.99,0.18) and (4.17,0.84) .. (6.56,1.97)   ;
\draw    (442,130.3) -- (509,130.3) ;
\draw [shift={(511,130.3)}, rotate = 180] [color={rgb, 255:red, 0; green, 0; blue, 0 }  ][line width=0.75]    (6.56,-1.97) .. controls (4.17,-0.84) and (1.99,-0.18) .. (0,0) .. controls (1.99,0.18) and (4.17,0.84) .. (6.56,1.97)   ;

\draw (153,53) node [anchor=north west][inner sep=0.75pt]   [align=left] {\begin{minipage}[lt]{74.18pt}\setlength\topsep{0pt}
\begin{center}
Contaminated \\data-generating\\distribution
\end{center}

\end{minipage}};
\draw (349,62) node [anchor=north west][inner sep=0.75pt]   [align=left] {\begin{minipage}[lt]{55.45pt}\setlength\topsep{0pt}
\begin{center}
Data-driven\\model
\end{center}

\end{minipage}};
\draw (512,62) node [anchor=north west][inner sep=0.75pt]   [align=left] {\begin{minipage}[lt]{67.35pt}\setlength\topsep{0pt}
\begin{center}
Out-of-sample\\environment
\end{center}

\end{minipage}};
\draw (48,121) node [anchor=north west][inner sep=0.75pt]    {$\mathbb{P}_{\star}$};
\draw (192,121) node [anchor=north west][inner sep=0.75pt]    {$\ {\mathbb{P}'}$};
\draw (372,121) node [anchor=north west][inner sep=0.75pt]    {$\ \ {\mathbb{P}}_{n}' \ $};
\draw (543,121) node [anchor=north west][inner sep=0.75pt]    {$\ \ \widetilde{\mathbb{P}}$};
\end{tikzpicture}
}
\vspace{-5pt}
\caption{Data-Driven Decision Making Cycle} 
\label{fig:decisioncycle}
\end{figure} 
\vspace{-5pt}

In the cycle depicted in Figure~\ref{fig:decisioncycle}, we usually collect $n$ i.i.d samples from the unknown data-generating distribution $\mathbb{P}'$, which may be identical to or distinct from the clean distribution $\mathbb{P}_{\star}$. Subsequently, we make a decision (e.g., parameter estimation) based on a model, ${\PP}_n'$, built from these samples. Such a model could be parametric or non-parametric.  These decisions are then put into action within the out-of-sample environment $\tilde{\mathbb{P}}$, which may or may not conform to the distribution $\mathbb{P}_{\star}$.
In this general cycle, the sample average method \eqref{eq:erm} may lead to poor out-of-sample guarantees. This motivates the use of alternative approaches. In the following paragraphs, we describe two of these approaches, DRO and robust statistics, by treating them as adversarial games. The crucial distinction lies in the \textbf{timing} of the contamination or attacks.

\textrm{(i)} (\textbf{DRO}: {\color{red}{$\widetilde{\PP} \neq {\PP}_\star ={\PP'}$}}) 
The attack occurs in the \textbf{post-decision} stage. In this scenario, the out-of-sample environment $\tilde{\mathbb{P}}$  diverges from the data-generating distribution $\mathbb{P}_{\star}$, and no contamination occurs before the decision, implying that $\mathbb{P}_{\star}=\mathbb{P}'$. For example, in adversarial deployment scenarios, malicious actors can deliberately manipulate the data distribution to compromise the performance of trained models. With full access to our trained machine learning model, the adversary endeavors to create adversarial examples specifically designed to provoke errors in the model's predictions.

To ensure good performance in terms of the optimal expected population loss over the out-of-sample distribution, the DRO framework introduces an uncertainty set $\mathcal{B}({\mathbb{P}}'_n)$ to encompass discrepancies between the in-sample-distribution ${\mathbb{P}}_n'$ and the out-of-sample distribution $\tilde{\PP}$. Subsequently, the DRO formulation minimizes the worst-case loss within this uncertainty set, thereby aiming to solve
\vspace{-3pt}
\begin{align}
\label{eq:our}
    \min_{\theta \in \Theta} \sup_{\QQ \in \mc B( \PP_n')} \EE_{\QQ}[\ell(\theta, Z)]. 
    \vspace{-5pt}
\end{align}

\textrm{(ii)} (\textbf{Robust Statistics}: {\color{blue}{$\mathbb{\widetilde{\mathbb{P}} \ =P}_{\star } \neq {\mathbb{P}}' $}})
The contamination occurs in the \textbf{pre-decision} stage.  
Many real-world datasets exhibit outliers or measurement errors at various stages of data generation and collection. In such scenarios, the observed samples are generated by a contaminated distribution $\mathbb{P}'$, which differs from the underlying uncontaminated distribution $\mathbb{P}_{\star}$. But, the out-of-sample distribution equals the original clean distribution. 
In contrast with DRO, the adversary corrupts the clean data according to a contamination model prior to training. Our objective is to clean the data during the training phase to achieve a robust classifier. It is noteworthy that the attacker does not have access to the specific model to be selected by the learner.

Given that the statistician knows that the data has been contaminated, a natural policy class to consider involves rectifying/correcting the contamination, and, for this, we introduce a rectification set $\mc R(\PP'_n)$ which models a set of possible pre-contamination distributions based on the knowledge of the empirical measure $\PP'_n$. To ensure good performance in terms of the optimal expected population loss over the clean distribution, the rectification/decontamination approach naturally induces the following min-min strategy: 
\vspace{-5pt}
\begin{align}
    \min_{\theta \in \Theta}\min_{\QQ \in\mc R ( {\PP}'_n)}\EE_{\QQ}\left[\ell(\theta, Z)\right].
    \label{eq:our2}
\end{align}
\vspace{-5pt}
Our goal in this paper is to develop \eqref{eq:our}, which is completely distinct from DRO.

\section{Automatic Outlier Rectification Mechanism}
\vspace{-8pt}
We now introduce our primary contribution by delineating \eqref{eq:our}. Our estimator is crafted to incorporate outlier rectification and parameter estimation within a unified optimization framework, facilitating automatic outlier rectification. A natural question arises from \eqref{eq:our}: how do we construct the rectification set in the space of probability distributions to correct the observed data set?
In this paper, we employ an optimal transport distance to create a ball centered at the contaminated empirical distribution $\mathbb{P}_n'$.
\begin{definition}[Rectification Set]
\label{defi:rect}
The optimal transport-based rectification set is defined as 
\begin{equation}
\mc R(\PP_n')= \{\QQ \in \mc P(\mc Z) :\mathds D_c(\QQ, \PP'_n) \leq \delta\},
\end{equation}
where $\delta>0$ is a radius, and $\mathds{D}_c$ is a specific optimal transport distance defined in Definition \ref{defi:ot}.
\end{definition}
\begin{definition}(Optimal Transport Distance~\citep{peyre2019computational,villani2009optimal}) 
\label{defi:ot}
Suppose that  $c(\cdot,\cdot):\mc Z \times \mc Z \rightarrow [0, \infty]$ is a lower semi-continuous cost function such that $c(z,z') =0$ for all $z,z'\in\mc Z$  satisfying $z=z'$. The optimal transport distance  between two probability measures $\PP, \QQ \in \mc P(\mc Z)$ is
\vspace{-8pt}
\[
\mathds D_c(\PP, \QQ) = \inf_{\pi \in \mc P(\mc Z \times \mc Z)}\Big\{\EE_{\pi}[c(Z, Z')] :\pi_Z  = \PP,\ \pi_{Z'} = \QQ\Big\}.
\]
Here, $\mathcal{P}(\mc Z \times \mc Z)$ is the set of joint probability distribution $\pi$ of $(Z,Z')$ supported on $\mathcal{Z} \times \mathcal{Z}$ while $\pi_Z $ and $\pi_{Z'}$ respectively refer to the marginals of $Z$ and $Z'$ under the joint distribution $\pi$.  
\end{definition}

If we consider $c(z,z') = \|z - z'\|^r$ as a metric defined on $\mathbb{R}^d$, where $r\in[1,\infty)$, then the distance metric $\mathds D_c^{1/r}(\PP, \QQ)$ corresponds to the $r$-th order Wasserstein distance~\citep{villani2009optimalWasserstein}.
In this paper, we pioneer the utilization of \textit{concave} cost functions, exemplified by $c(z,z') = \| z-z'\|^r$ where $r\in(0,1)$ in statistical robust estimation. We note that $\mathds D_c(\PP, \QQ)$ (as opposed to $\mathds D_c^{1/r}(\PP, \QQ)$ ) is also a metric for $r\in[0,1)$.  The rationale behind selecting a concave cost function is intuitive: it promotes what we colloquially refer to as \textit{long haul} transportation plans, enabling outliers to be automatically moved significant distances back towards the central tendency of data distribution. This, in turn, facilitates automatic outlier rectification. Concave costs promote long hauls due to their characteristic of exhibiting decreasing marginal increases in transportation cost. In other words, if an adversary decides to transport a data point by $\|\Delta\|$ units, it becomes cheaper to continue transporting the same point by an additional $\epsilon$ distance compared to moving another point from its initial location. We make further remarks about the connection between our estimator and prior work in Appendix \ref{app:estimator-prior-work}.

\textbf{Illustrative Example.} We provide empirical evidence showcasing the efficacy of the concave cost function on simulated data in Figure \ref{fig:regression-example-concave} and \ref{fig:regression-example-convex}. This example shows that our concave cost function is critical for moving the outliers (orange) properly to the bulk of the clean distribution (blue). For the concave cost, only the outliers are rectified (green), resulting in the proper line of best fit.

\vspace{-4pt}
\begin{figure}[!htbp]
    \centering
    \subfigure[$\delta=0$]{\includegraphics[width=0.32\textwidth]{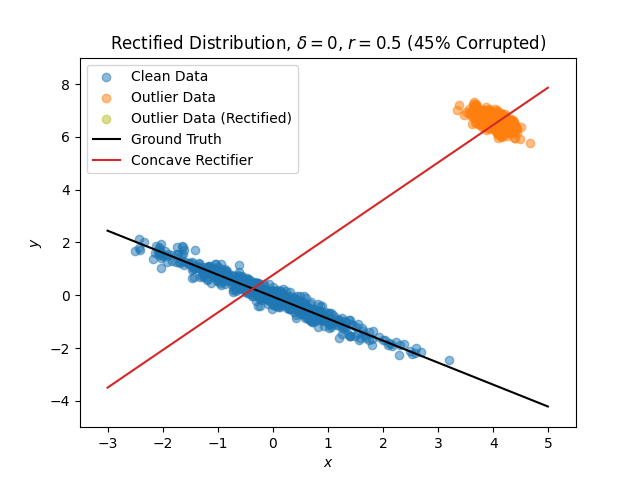}
    } 
    \subfigure[$\delta=0.9$]{\includegraphics[width=0.32\textwidth]{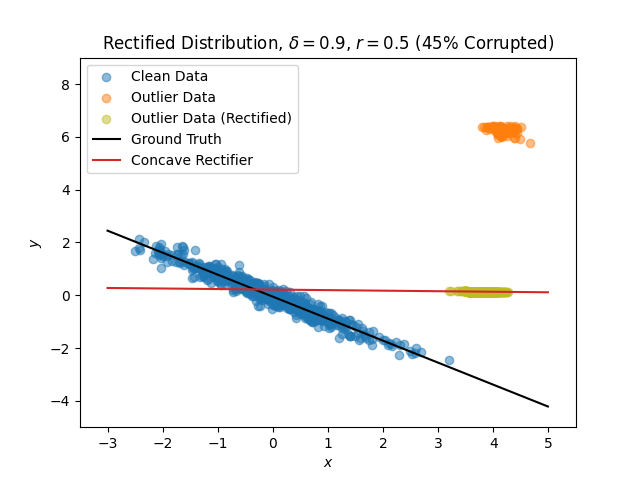}} 
    \subfigure[$\delta=1.5$]{\includegraphics[width=0.32\textwidth]{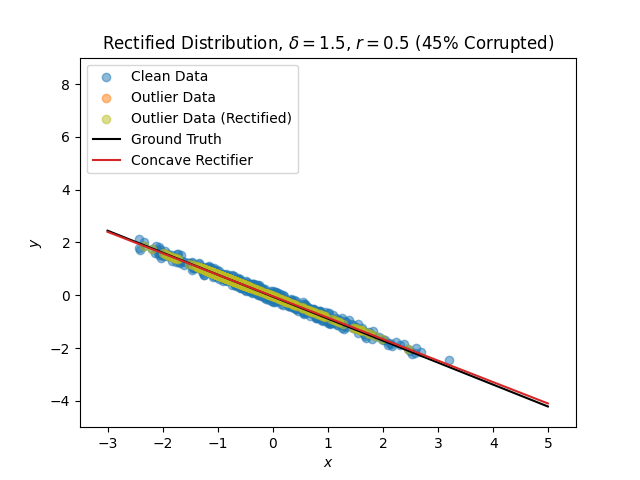}}
    \vspace{-8pt}
    \caption{The rectified data generated by our  estimator with \textbf{concave} cost function ($r=0.5)$.}
    \label{fig:regression-example-concave}
    \vspace{-2pt}
\end{figure}

\begin{figure}[!htbp]
    \centering
    \vspace{-5pt}
    \subfigure[$\delta = 0$]{\includegraphics[width=0.32\textwidth]{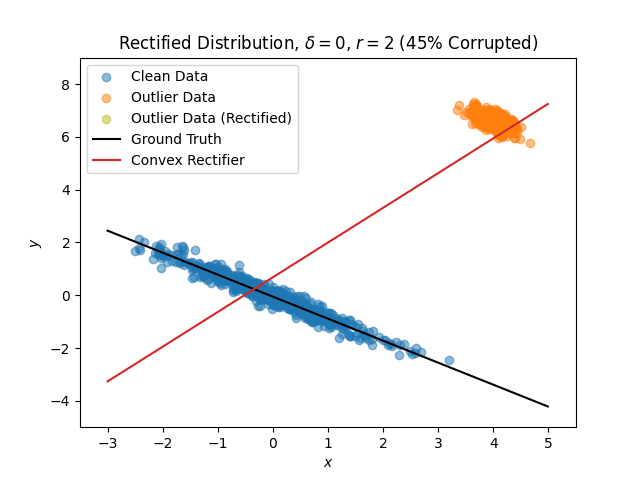}} 
    \subfigure[$\delta = 0.9$]{\includegraphics[width=0.32\textwidth]{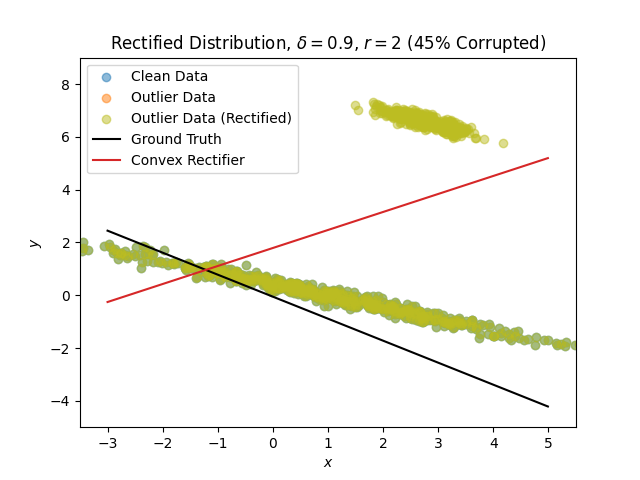}} 
    \subfigure[$\delta = 1.5$]{\includegraphics[width=0.32\textwidth]{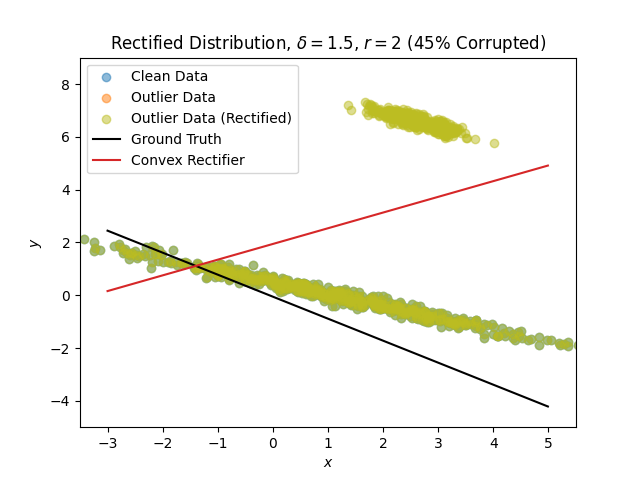}} 
    \vspace{-8pt}
    \caption{The rectified data generated by our estimator with \textbf{convex} cost function ($r=2)$.}
\label{fig:regression-example-convex}
\end{figure}
However, for the convex cost, regardless of the budget, all points both clean and corrupted are rectified towards each other instead, which results in an incorrect line of best fit. This illustrates succinctly the importance of our novel concave cost. More details on this example are given in Appendix \ref{app:illustrative-example}.

\vspace{-8pt}
\section{Reformulation Results}
\vspace{-8pt}

With this intuition in hand, we focus on the derivation of equivalent reformulations for the infinite-dimensional optimization problem over probability measures in \eqref{eq:our}. These reformulations provide us with a fresh perspective on adaptive quantile regression or estimation, where our introduction of an efficient transporter rectifies the empirical distribution, eliminating the influence of outliers.

To begin with, we can transform problem \eqref{eq:our} into an equivalent finite-dimensional problem by leveraging the following strong duality theorem.

\begin{proposition}[Strong Duality]
\label{prop:duality}
Suppose that $\ell(\theta,\cdot)$ is lower semicontinuous and integrable under $\PP_n'$ for any $\theta \in \Theta$. Then, the strong duality holds, i.e.,
\vspace{-3pt}
\begin{equation*}
\inf \limits_{\QQ \in \mc R(\PP'_n)} \EE_\QQ[\ell(\theta;Z)] = \max_{\lambda \ge 0} \E_{\PP'_n}\left[\min_{z\in\mc Z}\ell(\theta;z)
+\lambda c(z,Z')\right]-\lambda\delta.    
\end{equation*}
\end{proposition}
\vspace{-5pt}

The proof is essentially based on the strong duality results developed for Wasserstein distributionally robust optimization problems~\citep{zhang2022simple,litikhonov,blanchet2019quantifying,gao2022distributionally,mohajerin2018data}, which allows us to rewrite the original problem \eqref{eq:our} as $\inf_{\QQ \in \mc R(\PP'_n)} \EE_\QQ[\ell(\theta;Z)] = -\sup_{\QQ \in \mc R(\PP'_n)}\EE_\QQ[-\ell(\theta;Z)]$.

We proceed to examine several representative examples to better understand the proposed statistically robust estimator~\eqref{eq:our}. We begin with mean estimation to showcase our estimator's performance on one of the most classic problems of point estimation which can be easily understood. We then give an example for least absolute deviations (LAD) regression, one of the most important problems of robust statistics. LAD regression builds a conceptual foundation which leads into a discussion of our framework in more general cases and problem domains.

\vspace{-5pt}
\subsection{Mean Estimation}
\vspace{-5pt}

The mean estimation task is widely recognized as a fundamental problem in robust statistics, making it an essential example to consider. In this context, we define the loss function as $\ell(\theta;z) = \|\theta-z\|$. It is worth noting that when $\delta=0$, Problem \eqref{eq:our} is equivalent to the median, which has already been proven effective in the existing literature.
However, beyond the equivalence to the median, there are additional benefits to be explored regarding the proposed rectification set. By deriving the equivalent reformulation and analyzing the optimal rectified distribution, we can gain further insights into how the proposed statistically robust  estimator operates. This analysis provides valuable intuition into the workings of the estimator and its advantages.

\begin{theorem}[Mean Estimation]
\label{thm:mean}
Suppose that $\mc Z =\R^d$, $\ell(\theta;z) = \|\theta-z\|$ and the cost function is defined as $c(z,z')=\|z-z\|^r$ where $r\in(0,1)$. Without loss of generality, suppose that the  following condition holds 
\begin{equation}
\|\theta -z_1'\|\leq \|\theta -z_2'\|\leq \cdots\leq \|\theta -z_n'\|.
\label{eq:condition}
\end{equation}
Then, we have the inner minimization of \eqref{eq:our2} as
\begin{align}
& 
\max\left(\frac{1}{n}\sum_{i=1}^{k(\theta)-1} \|\theta-z_i'\|+\frac{1}{n}\left(1-\frac{n\delta-\sum_{i={k(\theta)+1}}^n \|\theta-z_i'\|^r}{\|\theta-z_{k(\theta)}'\|^{r}}\right)\|\theta-z_{k(\theta)}'\|,0\right).
\label{eq:mean}
\end{align}
where $k(\theta):=\max\limits_{k\in[n]} \left\{k:\tfrac{1}{n}\sum_{i=k}^n\|\theta-z_i'\|^r \ge \delta \right\}$. 
\end{theorem}
\vspace{-10pt}
We give the proof in Appendix \ref{app:proof}.

\begin{remark}
The resulting reformulation problem can be viewed as finding a quantile of $\|\theta-z'\|$ controlled by the budget $\delta$. 
If we have a sufficient budget $\delta$ such that $\mathbb{E}_{\mathbb{P}'_n}[\|\theta-Z'\|^r]\leq \delta$, it implies that all data points have been rectified to the value of $\theta$. Consequently, the minimum value of $\min_{\mathbb{Q} \in \mathcal{R}(\mathbb{P}_n')} \mathbb{E}_{\mathbb{Q}}[\ell(\theta,Z)]$ will be equal to zero. In nontrivial cases, when given a budget $\delta>0$ and the current estimator $\theta$, our objective is to identify and rectify the outliers in the observed data. To achieve this, we start by sorting the data points based on their loss value $\|\theta-z_i'\|$. 
We relocate the data points starting with the one having the largest loss value, $z_n'$. The goal is to move each data point towards the current mean estimation $\theta$ until the entire budget is fully utilized.
\end{remark}

Building upon the proof of Theorem \ref{thm:mean}, we can establish the following characterization of the rectified distribution. The concave cost function $\|z-z'\|^r$ (with $r\in(0,1)$) plays a pivotal role in this context by endowing the perturbation with a distinctive long haul structure. Intuitively, for each data point, we can only observe two possible scenarios: either the perturbation is zero, indicating no movement or the data point is adjusted to eliminate the loss. In this process, the rectified data points are automatically identified as outliers and subsequently rectified.

\begin{proposition}[Characterization of Rectified Distribution] 
\label{prop:rectified-distribution}
Assuming the same conditions as stated in Theorem \ref{thm:mean}, we can conclude that the optimal distirbution $\QQ^\star$
\[\QQ^\star (dz) = \frac{1}{n}\sum_{i=1}^{k(\theta)-1}\delta_{z_i'}(dz)+\frac{\eta}{n}\delta_{z'_{k(\theta)}}+\frac{n-k(\theta)+1-\eta}{n}\delta_{\theta}(dz)\]
where $\eta = 1-\frac{n\delta-\sum_{i={k(\theta)+1}}^n \|\theta-z_i'\|^r}{\|\theta-z_{k(\theta)}'\|^{r}}$.
\end{proposition}
\begin{remark}
The existence of optimal solutions follows directly from \citep[Theorem 2]{yue2022linear}.
\end{remark}

\vspace{-5pt}
\subsection{Least Absolute Deviation (LAD) Regression and More General Forms}
\vspace{-5pt}
\label{sec:examples}

Following the same technique, we can also derive the least absolute deviation case. 

\begin{theorem}[LAD Regression]
\label{prop:quantile}
Suppose that $\mc Z := \mc X \times \mc Y = \R^{d+1}$, $\ell(\theta,z) = \|y-\theta^Tx\|$ and the cost function is defined as $c(z,z')=\|z-z\|^r$ where $r\in(0,1)$ and $\|\cdot\|$ is the $\ell_2$ norm. Without loss of generality, suppose that $\|y_1'-\theta^Tx_1'\|\leq \|y_2'-\theta^Tx_2'\|\leq \cdots\leq \|y_n'-\theta^Tx_n'\|$, we have the inner minimization of \eqref{eq:our2} as
\begin{small}
\begin{align*}
& \max \left(\frac{1}{n}\sum_{i=1}^{k(\theta)-1} \|y_i'-\theta^Tx_i'\|^r+\frac{1}{n}\left(1-\frac{n\delta'-\sum_{i={k(\theta)+1}}^n \|y_i'-\theta^Tx_i'\|^r}{\|y_{k(\theta)}'-\theta^Tx_{k(\theta)}'\|^r }\right)\|y_{k(\theta)}'-\theta^Tx_{k(\theta)}'\|,0\right),
\end{align*}
\end{small}
where $k(\theta):=\max\limits_{k\in[n]} \left\{k:\tfrac{1}{n}\sum_{i=k}^n\|y_k'-\theta^Tx_k'\|^r \ge \delta' \right\}$ and $\delta' = \delta\|(\theta,-1)\|^r$.
\end{theorem}
\vspace{-5pt}

We give the proof of Theorem \ref{prop:quantile} in Appendix \ref{app:proof}. The structure of the optimal rectified distribution also resembles that of Proposition \ref{prop:rectified-distribution}: the rectified data points are shifted towards the hyperplane $y=\theta^Tx$ that best fits the clean data points. For a more detailed explanation of the long haul structure and extensions to more general cases, interested readers are encouraged to refer to Appendix \ref{app:estimator-prior-work}.

\vspace{-5pt}
\subsection{Computational Procedure}
\vspace{-5pt}

Our procedure is motivated by the empirical efficacy of subgradient descent for training deep neural networks, as described in e.g. \citep{li2020understanding}. In each iteration, given the current estimate $\theta$, we compute the optimal rectified distribution. For simple applications such as mean estimation and least absolute deviation regression, we can solve this optimization problem efficiently using the quick-select algorithm or by utilizing an existing solver for linear programs to obtain the optimal solution for the dual variable $\lambda$, which we denote $\lambda^*$. Then, we employ the subgradient method on the rectified data, iterating until convergence. This optimization process automatically rectifies outliers using a fixed budget $\delta$ across all iterations. As the value of $\theta$ changes, the same budget $\delta$ for all iterations results in varying the quantile used for selecting outliers. Thus, our estimator can be regarded as an iteratively adaptive quantile estimation approach.

We now concretely detail the procedure for LAD regression. The procedure for mean estimation is analogous and results from a simple change in the loss function. We must start by initially addressing the computation of the optimal dual variable $\lambda^\star$ for the inner minimization over the probability space. Without loss of generality, we recall that this problem is
\vspace{-5pt}
\begin{equation}
\label{eq:opt_lambda}
    \max_{\lambda \ge 0} \frac{1}{n}\sum_{i=1}^n \min\left\{\|\tilde\theta^Tz_i'\|,\frac{\lambda \|\tilde\theta^Tz_i'\|^r}{\|\tilde\theta\|_*^r}\right\}-\lambda\delta.
\end{equation}

\vspace{-3pt}
We show in the Appendix in Section \ref{app:proof} by Lemma \ref{lemma:lambda} that this problem can be solved by applying the quick-select algorithm to find the optimal $\lambda^\star$ and the knot point $k(\theta_t)$ which demarcates estimated outliers. Applying this lemma yields our estimation approach, which is described in Algorithm \ref{alg:1}.

\vspace{-5pt}
\begin{algorithm}[hbt!]
    \caption{Statistically Robust Estimator}\label{alg:1}
    \KwData{Observed data $\{z_i'\}_{i=1}^n$, initial point $\theta_0$, stepsizes $\alpha_t>0$;}
    \For{$t=0,\cdots,T$}{
        1. \textbf{Sort} the observed data $\{z_i'\}_{i=1}^n$ via the value $\|y_i'-\theta_t^Tx_i'\|$. \\
     2. \textbf{Quick-Select} algorithm to get the knot point $k(\theta_t)$ and the optimal $\lambda^\star$;\\
     3. \textbf{Subgradient step} on the detected clean data:
     \vspace{-10pt}
     \[\theta_{t+1} =\theta_{t}- \frac{\alpha_t}{n}\sum_{i=1}^{k(\theta_t)}\textrm{sgn}(\theta_t^Tx_i'-y_i')\cdot x_i' \]
     \vspace{-10pt}
    }
\end{algorithm} 
\vspace{-5pt}
Additional details on an alternative linear program approach for estimating $k(\theta_t)$ and $\lambda^*$ are given in Appendix \ref{app:algorithm-alternative}. We also propose a procedure for estimating more general regression models and limitations in Appendix \ref{app:algorithm-general}. Further details on the optimization problem are given in Appendix \ref{app:algorithm-details}.

\vspace{-5pt}
\section{Experimental Results}
\vspace{-5pt}
In this section, we demonstrate the effectiveness of the proposed statistically robust estimator through various tasks: mean estimation, LAD regression, and two applications to volatility surface modeling.

\vspace{-5pt}
\subsection{Mean Estimation and LAD Regression Experiments}
\vspace{-5pt}

We perform simulation experiments with mean estimation and least absolute deviation linear regression for our estimator to illustrate its efficacy. For mean estimation, we compare our estimator with the mean, median, and trimmed mean estimators. We observe that our estimator outperforms all other methods, despite providing the oracle corruption level $\epsilon$ to the trimmed mean estimator. For LAD regression, we compare our estimator with the OLS, LAD, and Huber regression estimators. We find that, again, our estimator outperforms all other methods. A complete overview of these two experiments can be found in Appendix \ref{app:experiments-overview}.

\vspace{-5pt}
\subsection{Options Volatility Experiments}
\vspace{-5pt}
\label{sec:options}

In this section, we conduct empirical studies on real-world applications in fitting and predicting option implied volatility surfaces. 
Options are financial instruments allowing buyers the right to buy (call options) or sell (put options) an asset at a predetermined price (strike price) by a specified expiration date. In this study, we focus on European-style options, which can only be exercised at expiration. Option prices are influenced by the volatility of the underlying asset's price. Implied volatility (IV), derived from an option's market price, indicates the market's volatility expectations. The implied volatility surface (IVS) represents the variation of IV across different strike prices and times to maturity. Accurate IVS modeling is crucial for risk assessment, hedging, pricing, and trading decisions. However, outliers in IV can distort the IVS, necessitating robust estimation methods. Our approach addresses this by estimating the IVS in the presence of outliers. We conduct two experiments to demonstrate the effectiveness of our statistically robust estimator: (a) using a kernelized IVS estimator and (b) using a state-of-the-art deep learning IVS estimator.

\vspace{-5pt}
\subsubsection{Kernelized Volatility Surface Estimation}
\vspace{-5pt}
\label{sec:options-kernel}

\paragraph{Data.} 
Our data set comprises nearly 2,000 option chains containing daily US stock option prices from 2019--2021. The options data is sourced from WRDS, a financial research service available to universities. The option chains were identified as containing significant outliers by our industry partner, a firm providing global financial data and analysis. We were blinded to this choice. We randomly draw a training set and test set from each chain and estimate surfaces. We assess the surfaces' out-of-sample performance using mean absolute percentage error (MAPE) and the discretized surface gradient (defined as $\dsg$). MAPE evaluates the error of the surface versus observed implied volatilities. $\dsg$ evaluates the smoothness of the surface. Further details can be found in Appendix \ref{app:options-details}.

\vspace{-5pt}
\paragraph{Benchmarks.} 
We compare our estimator developed in Theorem \ref{prop:quantile} to two benchmarks. The first is kernel smoothing (denoted ``KS''), a well-established method for estimating the IVS \citep{ait1998nonparametric}. In KS, each point on the IVS is constructed by a weighted local average of implied volatilities across nearby expiration dates, strike prices, and call/put type. The weights are given by a kernel matrix which depends on these three features of an option (see next paragraph). The second benchmark is a two-step ``remove and fit'' kernel smoothing method (denoted ``2SKS'') which first attempts to remove outliers via Tukey fences  \citep{hoaglin1986performance} before applying the KS method.

\vspace{-5pt}
\paragraph{Kernelized Regression Problem.} We now describe the kernel and our estimator. Following the convention in \citep{ait1998nonparametric}, the $i$th option in a chain is featurized as the vector $\R^3 \ni x_i'= (\log\tau_i, u(\Delta_i), \mathbb{I}_{\mathrm{call}(i)})$, where $\tau_i$ is the number of days to the option's expiration date, $\Delta_i \in [-1,1]$ is a relative measure of price termed the Black-Scholes delta, $u(x) = x\mathbb{I}_{x \geq 0} + (1+x)\mathbb{I}_{x < 0}$, and $\mathbb{I}_{\mathrm{call}(i)}$ is 1 if the option is a call option and 0 if the option is a put option. The goal is to fit the pair ($y_i'$: the implied volatility of option $i$, $x_i'$: the features) via a kernelized regression model.
We choose a Gaussian-like kernel $K_h(x,x')$ to measure the distance between options $x$ and $x'$ as $
    K_h(x,x') = \exp(-\|{(x-x')}/{2h}\|^2)
$
where division of $x-x'$ by $h$ is element-wise for a vector of bandwidths $h \in \R^3_+$. When the budget $\delta =0$, we want to solve
$
    \min_{\theta\in \R^n}\sum_{i=1}^n v_i |y_i'-\theta^TK_h^i|,
$
where $v_i$ is the $i$-th entry of a vector of vegas for $\{x'_i\}_{i=1}^n$ and $K_h^i$ is the $i$-th row of the kernel matrix. The implied volatilities are weighted by vega $v_i$ to improve surface stability as in \citep{mayhew1995volatility}. 
We conducted a comparison between our estimator and the benchmarks. We note that the KS method can be regarded as a standard kernelized least square approach \citep{hansen2022econometrics}.

\vspace{-10pt}
\paragraph{Results.} Our approach improves upon the benchmark approach in both MAPE and $\dsg$ on the out-of-sample test sets, showcasing the importance of jointly estimating the rectified distribution and model parameters with our estimator. The results of our experiment are displayed in Table \ref{tab:options-stats}. We compare the KS and 2SKS benchmarks against our estimator with both fixed $\delta=0.01$ (chosen to correspond to a 1\% change in volatility) and $\delta$ chosen by cross-validation. More details can be found in Appendix \ref{app:options-details}. Our estimator outperforms others with a mean MAPE of 0.225, compared to 0.294 for KS and 0.236 for 2SKS, showing a 23\% and 5\% improvement respectively. For surface smoothness, our estimator achieves a mean DSG of 6.5, significantly better than 20.2 for KS and 7.5 for 2SKS, indicating 67\% and 13\% improvements. Notably, our industry partner applied our estimator to an unseen test set collected after this paper was written and was able to release to production 25\% more surfaces than they had before when using their existing proprietary method.

\vspace{-3pt}
\begin{table}[h]
    \label{tab:options-stats}
    \centering
    \caption{Results of our experiment with kernelized IVS estimation.}
    \vspace{-5pt}
    \scalebox{0.85}{
    \begin{tabular}{ccccccc} 
        \toprule
        Model MAPE & $0.5\%$ Quantile & $5\%$ Quantile & Median & Mean & $95\%$ Quantile & $99.5\%$ Quantile \\
        \midrule
        KS & 0.026 & 0.068 & 0.232 & 0.294 & 0.677 & 1.438 \\
        2SKS & 0.026 & 0.056 & 0.172 & 0.236 & 0.602 & 1.389 \\
        Ours ($\delta=10^{-2}$) & 0.028 & 0.057 & 0.170 & 0.225 & 0.535 & 1.207 \\
        Ours (CV) & 0.028 & 0.057 & 0.169 & 0.224 & 0.534 & 1.240 \\
        \midrule
        Model $\nabla\hat{S}$ & $0.5 \%$ Quantile & $5\%$ Quantile & Median & Mean & $95\%$ Quantile & $99.5\%$ Quantile \\
        \midrule
        KS & 0.313 & 1.606 & 14.434 & 20.188 & 57.797 & 108.901 \\
        2SKS & 0.050 & 0.124 & 2.229 & 7.491 & 33.276 & 78.144 \\
        Ours ($\delta=10^{-2}$) & 0.048 & 0.121 & 1.898 & 6.502 & 28.064 & 66.410 \\
        Ours (CV) & 0.043 & 0.109 & 1.513 & 5.079 & 21.096 & 55.005 \\
        \bottomrule
    \end{tabular}
    }
\end{table}

\begin{remark}
To contextualize these results, we note that the difference in MAPE of a surface and an option chain containing outliers and the MAPE of a surface and an option chain containing no outliers will not be large if the set of outliers is small. Consider the MAPE of the same surface for two different option chains of size $n$, $O_1$ and $O_2$, where a small fraction $k/n$ of the options of $O_2$ are outliers. Supposing the modal APE is $0.3$ and the outlier APE is a considerable $1.0$, for an options chain with $n=50$ and just $k=5$, the MAPE difference will be only $\frac{k + 0.3(n-k)}{n} - 0.3 = 0.07$.    
\end{remark}

We perform an additional experiment with the same dataset in Appendix \ref{app:options-additional-experiments} which demonstrates the usefulness of our method for estimating an IVS for use on the trading day after the initial contaminated surface is observed. We find similar outperformance of our method versus the benchmark methods. 

\vspace{-5pt}
\subsubsection{Deep Learning Volatility Surface Estimation}
\vspace{-5pt}
\label{sec:deep-learning}

In this section, we apply our statistically robust estimator to state-of-the-art deep learning prediction approaches for modeling volatility developed in \cite{chataigner2020deep}. In this work, deep networks are used to estimate local volatility surfaces. The \textit{local volatility} or \textit{implied volatility function} model introduced by \cite{rubinstein1994implied}, \cite{dupire1994pricing}, and \cite{derman1996local} is a surface which allows for as close of a fit as possible to market-observed implied volatilies without arbitrage, which is of interest to many market participants (smoothing methods, in contrast, do not make this guarantee). Further background and details for this section are available in Appendix \ref{app:options-deep-learning}.

We select the data set from \cite{chataigner2020deep} consisting of (options chain, surface) pairs from the German DAX index. We first contaminate the chains by replacing an $\epsilon$ fraction of each price $p$ with $10p$. We then test our estimated surface against the true surface. To estimate price and volatility surfaces under data corruption, we applied our statistically robust estimator to the benchmark approach. We use the Dupire neural network of \cite{chataigner2020deep} as a benchmark, which estimates the surface under local volatility model assumptions and no-arbitrage constraints. This method enforces these conditions on the surface using hard, soft, and hybrid hard/soft constraints.

We evaluate our robust estimator using the same metrics as \cite{chataigner2020deep}, RMSE and MAPE, repeated over three trials with different random seeds. 
Our approach outperforms the baseline approach across all averages, and does so more clearly as the corruption level increases, despite the strong regularizing effect of the no-arbitrage constraints and enforcement of Dupire's formula. Our improvement is most impactful for the most accurate model utilizing soft arbitrage constraints. For this model, the test set RMSE and MAPE are reduced by 33\% and 34\%. This experiment displays the efficacy of our estimator in a state-of-the-art deep learning approach to volatility modeling.

\begin{table}[htbp!]
    \label{tab:options-stats-dl}
    \centering
    \caption{Results of our experiment with deep learning surface estimation.
    }
    \vspace{-5pt}
    \scalebox{0.8}{
    \begin{tabularx}{14cm}{@{}c *5{>{\centering\arraybackslash}X}@{}}
    \toprule
    \toprule
    \multicolumn{4}{c}{
    \textbf{Panel A: Results by Constraint Type}} \\
    \midrule
    \midrule
    Model & Hard Constraints & Hybrid Constraints & Soft Constraints  \\
    \midrule
    Dupire NN RMSE	& 0.125 & 0.140 & 0.044 \\
    Our RMSE 	    & 0.110 & 0.131 & 0.029 \\
    \midrule
    Dupire NN MAPE	& 0.343 & 0.546 & 0.111 \\
    Our MAPE	    & 0.310 & 0.508 & 0.074 \\
  \end{tabularx}
  }
  \scalebox{0.8}{
  \centering
  \begin{tabularx}{14cm}{@{}c *4{>{\centering\arraybackslash}X}@{}}
    \toprule
    \toprule
    \multicolumn{4}{c}{
    \textbf{Panel B: Results by Corruption Level}} \\
    \midrule
    \midrule
    Corruption Level & $\epsilon=20\%$ & $\epsilon=30\%$ & $\epsilon=40\%$ \\
    \midrule
    Dupire NN RMSE	& 0.077 & 0.091 & 0.168 \\
    Our RMSE 	    & 0.075 & 0.084 & 0.141 \\
    \midrule
    Dupire NN MAPE	& 0.061 & 0.070 & 0.115 \\
    Our MAPE	    & 0.060 & 0.066 & 0.097 \\
    \bottomrule
  \end{tabularx}
  }
\end{table}

\vspace{-5pt}
\section{Conclusion}
\vspace{-5pt}
In conclusion, we propose an automatic outlier rectification mechanism that integrates outlier correction and estimation within a unified optimization framework. Our novel approach leverages the optimal transport distance with a concave cost function to construct a rectification set within the realm of probability distributions. Within this set, we identify the optimal distribution for conducting the estimation task. Notably, the concave cost function's "long hauls" attribute facilitates moving only a fraction of the data to distant positions while preserving the remaining dataset, enabling efficient outlier correction during the optimization process.
Through comprehensive simulation and empirical analyses involving mean estimation, least absolute regression, and fitting option implied volatility surfaces, we substantiate the effectiveness and superiority of our method over conventional approaches. This demonstrates the potential of our framework to significantly enhance outlier detection integrated within the estimation process across diverse analytical scenarios.

\bibliography{arxiv}
\bibliographystyle{plainnat}
\newpage

\appendix
\section*{Appendix}
\appendix

\section{Organization of the Appendix}
\label{app:appendix-organization}

We organize the appendix as follows:
\begin{itemize}
    \item The organization of the Appendix is given in this section (Section \ref{app:appendix-organization}).
    \item Additional details on the connection between our estimator and prior work are given in Section \ref{app:estimator-prior-work}.
    \item Additional explanation of the long haul structure induced by our estimator and extensions to more general forms of problems are given in Appendix \ref{app:estimator-more-details}.
    \item The proof of Theorem \ref{thm:mean} and Theorem \ref{prop:quantile} are given in Section \ref{app:proof}.
    \item The long haul structure for linear regression with a concave cost function is provided in Section \ref{app:proof-regression-concave-details}.
    \item The detailed computational procedure for mean estimation, linear absolute regression, and more general cases can be found in Section \ref{app:algorithm}.
    \item The main experiments for mean estimation and LAD regression are described in Appendix \ref{app:experiments-mean-estimation} and \ref{app:experiments-lad-regression} respectively. Additional details are given in the following appendices.
    \item Additional details on the illustrative example in the text is given in Section \ref{app:illustrative-example}.
    \item Additional details on mean estimation simulations are given in Section \ref{app:mean-estimation-details}.
    \item Additional plots for the mean estimation simulations for the concave cost function are given in Section \ref{app:simulation-mean-estimation-concave}.
    \item Additional plots for the mean estimation simulations for the convex cost function are given in Section \ref{app:simulation-mean-estimation-convex}.
    \item Additional details on LAD regression simulations are given in Section \ref{app:lad-regression-details}.
    \item Additional plots for the LAD regression simulations for the concave cost function are given in Section \ref{app:simulation-regression-concave}.
    \item Additional plots for the LAD regression simulations for the convex cost function are given in Section \ref{app:simulation-regression-convex}.
    \item Additional option volatility surface experiment details are given in Section \ref{app:options-additional-experiments}.
    \item Additional information on the volatility surface data set and losses are given in Section \ref{app:options-details}.
    \item Background and additional details on the deep learning volatility modeling experiment are given in Section \ref{app:options-deep-learning}.
\end{itemize}

\newpage
\section{Estimator Remarks}

\subsection{Connections to Prior Work}
\label{app:estimator-prior-work}

Additional details on connections to prior work are given in the following remark:

\begin{remark}
    \textrm{(i)} When we disregard the outer minimization concerning estimation parameters, the inner minimization problem over probability measures is related to the minimum distance functionals-based estimator, initially proposed by \citep{donoho1988automatic}. The estimator is obtained by first projecting the corrupted distribution $\PP'_n$ onto a family of distributions $\mathcal{G}$ under a certain distribution discrepancy measure $\mathds D$. Then, the optimal parameters are selected for the resulting distribution. However, even with additional information about our contamination models, this ``two-stage'' procedure  has two major drawbacks: the difficulty in choosing an appropriate family of distributions $\mathcal{G}$ and the inherent computational challenge to project probability distributions.  Moreover, the first stage (projection) is usually sensitive to the choice of family $\mc G$. In contrast, we propose a novel approach that integrates outlier rectification (i.e., explicit projection) and parameter estimation into a joint optimization framework in \eqref{eq:our}. \textrm{(ii)} This min-min strategy has also been explored in \cite{jiang2024distributionally} using an artificially constructed rectification set. Their primary focus is to utilize the min-min formulation to recover well-studied robust statistics estimators. In contrast, our focus is on a new conceptual framework for formulating novel estimators which have not been studied before.
\end{remark}

\subsection{Long Haul Structure and Extensions}
\label{app:estimator-more-details}

An explanation of the long haul structure and extensions to more general forms are given in this subsection.

The proposed rectification set has the potential for broader applicability across various applications and problem domains.  This versatility makes the rectification set a valuable tool for addressing and improving solutions in diverse problem settings in the future. 

{\rm (i) (\textbf{Concave cost function})} 
From Proposition \ref{prop:duality}, the optimal rectification admits
$z_{\textrm{best}} \in \arg\min_{z\in\mc Z}\ell(\theta,z')
+\lambda c(z,z').$
Suppose we select a concave cost function $c$ that grows strictly slower than the loss function $\ell$. In this scenario, when the budget is small, the rectified data point $z_{\textrm{best}}$ consistently exhibits the long haul structure, ensuring automatic outlier identification properties. This implies that the rectification process effectively identifies outliers, as the influence of the cost function dominates over the loss function for small budgets. To further illustrate the concept, we can consider a linear regression problem with the squared loss function $\ell(\theta,z) = (y - \theta^Tx)^2$ or a nonlinear loss using the $r$-th norm as the cost function. We refer the readers to Appendix \ref{app:proof-regression-concave-details}. 

{\rm (ii) (\textbf{Other distribution metric/discrepancy})} Based on the two examples discussed earlier, we can infer that the optimal transport distance-based rectification set is particularly suitable for regression or problems with continuous response variables. This is because the budget $\delta$ is used to compensate for the loss caused by identified outliers. However, in classification tasks, where the loss function may be less sensitive, the effectiveness of the optimal transport distance-based approach may be limited.
To address this limitation and complement the rectification set for classification tasks, an alternative approach is to use $\phi$-divergences. This approach has already been proposed in a study by \citep{chen2022rockafellian,antil2024rockafellian} for handling outliers in image classification tasks. 
The $\phi$-divergence-based rectification set operates by adjusting the weight assigned to outliers, effectively minimizing their impact on the overall data set. By reducing the weight placed on the detected outliers, the rectification set aims to remove their influence from the data set. This selective adjustment allows for the identification and removal of outliers, leading to a more refined and reliable data set.

\newpage
\section{Proof Details}

\subsection{Proof of Theorem \ref{thm:mean} and Theorem \ref{prop:quantile}.}
\label{app:proof}
To start with, we give two crucial lemmas \ref{lemma:1} and \ref{lemma:lambda}.
\begin{lemma}
\label{lemma:1}
Suppose that $a,b,\lambda >0$ and  $r\in(0,1)$, we have 
\[
\min_{x\in[0,a/b]} a-bx + \lambda x^r  = \min\left\{a, \frac{\lambda a^r}{b^r}\right\}. 
\]
\end{lemma}
\begin{proof}
The argument is easy. The function $g(x) =  a-bx + \lambda x^r$ is concave as 
the second derivative is always negative:
\[
\nabla^2 g(x) = \lambda r(r-1)x^{r-2}
\]
for all $x\ge0$. 
\end{proof}

\begin{lemma}
\label{lemma:lambda}
Suppose that there is an increasing sequence $0\leq x_1 <  x_2 < \cdots <x_n$. The optimal solution of 
\[\max_{\lambda \ge 0}  \sum_{i=1}^n \alpha_i \min\{x_i, \lambda x_i^r\}-\lambda \delta\]
is $\lambda^\star = x_k^{1-r}$ where 
\[k:= \mathop{\max}_{k\in[n]} \left\{k:\sum_{i=k}^n \alpha_i x_k^r\ge \delta\right\}.\]
Here $\alpha_i\in[0,1],\sum_{i=1}^n \alpha_i= 1$, $\delta >0$ and $r\in(0,1)$. Moreover, the optimal function value is 
\[\max(\sum_{i=1}^{k-1} \alpha_ix_i +\left(1-\frac{(\delta -\sum_{i=k+1}^n \alpha_i x_i^r)}{\alpha_k x_k^r}\right)\alpha_k x_k,0).\]
\end{lemma}
\begin{proof}
Without loss of generality, we assume that $0 < x_1 < x_2 < \dots < x_n$ as the zero part of the sequence does not affect the result.

First, given a fixed $\lambda \in \mathbb{R}+$ and the inequality $x_{t+1} <\lambda x_{t+1}^r$, our goal is to show that $x_t < \lambda x_t^r$. Let $x_t = \eta x_{t+1}$, where $\eta \in (0,1]$. Then, we have
\begin{align*}
x_{t} = \eta x_{t+1} < \eta \lambda x_{t+1}^r = \frac{\lambda}{\eta^r} \eta x_{t}^r = \lambda \eta^{1-r} x_{t}^r\leq  \lambda x_{t}^r. 
\end{align*}
Considering a fixed $\lambda \in \mathbb{R}_+$, we can express the objective function as:
\[
 \sum_{i=1}^n \alpha_i \min\{x_i, \lambda x_i^r\}-\lambda \delta =  \sum_{i=1}^{k-1} \alpha_ix_i+ \alpha_k \min\{x_k, \lambda x_k^r\}+\lambda\sum_{i=k+1}^n \alpha_i x_i^r-\lambda \delta
\]
where $1 \leq k \leq n$ and $\alpha_i$ are weights associated with each $x_i$. Two cases arise:
\begin{enumerate}
\item For all $1 \leq t \leq k-1$, $x_t < \lambda x_t^r$, or for all $k+1 \leq t \leq n$, $x_t > \lambda x_t^r$.
\item At the $k$-th point, we have $x_k = \lambda x_k^r$.
\end{enumerate}

Since we are dealing with a concave piecewise linear function, the optimal $\lambda$ will be a knot point (case 2). Otherwise, modifying $\lambda$ would increase the objective value. Therefore, we can establish the first-order optimality condition as follows:
\[
 \sum_{i=k+1}^n \alpha_i x_i^r + \mu \alpha_k x_k^r = \delta. 
\]
where $k$ and $\mu \in [0,1]$ are determined through a search using the quick-select algorithm. Consequently, the optimal solution is obtained as $\lambda^* = x_k^{1-r}$ and 
\[
k = \max_{k\in [n]}\left\{\sum_{i=k}^n \alpha_i x_i^r\ge\delta \right\}. 
\]
Thus, we further get $\mu = \frac{(\delta -\sum_{i=k+1}^n \alpha_i x_i^r)}{\alpha_k x_k^r}$ and the optimal function value admits 
\begin{align*}
& \, \sum_{i=1}^{k-1} \alpha_ix_i+ \alpha_k \min\{x_k, \lambda x_k^r\}+\lambda\sum_{i=k+1}^n \alpha_i x_i^r-\lambda \delta \\
= & \, \sum_{i=1}^{k-1} \alpha_ix_i+ \alpha_k \min\{x_k, \lambda x_k^r\}-\lambda \mu \alpha_k x_k^r \\
= & \, \sum_{i=1}^{k-1} \alpha_ix_i+(1-\mu)\alpha_k x_k \\
= & \, \sum_{i=1}^{k-1} \alpha_ix_i +\left(1-\frac{(\delta -\sum_{i=k+1}^n \alpha_i x_i^r)}{\alpha_k x_k^r}\right)\alpha_k x_k 
\end{align*}

Based on our discussion, we exclude the corner case where $k=1$. Therefore, $\delta$ is sufficiently large such that $\sum_{i=1}^n \alpha_i x_i^r \leq \delta$. It is evident that in this trivial scenario, the optimal function value is zero.

We conclude our proof. 
\end{proof}

We now give the proof of Theorem \ref{thm:mean}:

\begin{proof}
Before we prove the theorem, it is worth highlighting that for any fixed $\theta$, we can always sort $\{z_i\}_{i=1}^n$ based on the error $\|\theta - z_i'\|$ to satisfy condition~\eqref{eq:condition}.

By the strong duality result in Proposition \ref{prop:duality}, we have 
\begin{align*}
& {\color{white}{=}} \min \limits_{\QQ \in \mc B(\PP'_n)} \EE_\QQ[\|\theta-Z'\|] \\
& =  \max_{\lambda \ge 0} \E_{\PP'_n}\left[\min_{\Delta \in \R}\|\theta-Z'-\Delta\|
+\lambda\|\Delta\|^r\right]-\lambda\delta\\
& = \max_{\lambda \ge 0} \E_{\PP'_n}\left[\min_{0\leq \|\Delta\|\leq \|\theta-Z'\|}\|\theta-Z'\|-\|\Delta\|
+\lambda\|\Delta\|^r\right]-\lambda\delta\\
&=\max_{\lambda \ge 0} \E_{\PP'_n}[\min\left\{ \|\theta-Z'\|,\lambda \|\theta-Z'\|^r\right\}]-\lambda\delta\\
& = \max_{\lambda \ge 0} \frac{1}{n}\sum_{i=1}^n \min\{\|\theta-z_i'\|,\lambda\|\theta-z_i'\|^r\}-\lambda\delta\\
& =\max\left(\frac{1}{n}\sum_{i=1}^{k(\theta)-1} \|\theta-z_i'\|+\frac{1}{n}\left(1-\frac{n\delta-\sum_{i={k(\theta)+1}}^n \|\theta-z_i'\|^r}{\|\theta-z_{k(\theta)}'\|^{r}}\right)\|\theta-z_{k(\theta)}'\|,0\right)
\end{align*}
where the third equality follows from Lemma \ref{lemma:1} in Appendix \ref{app:proof} and the fifth one is due to Lemma \ref{lemma:lambda} in Appendix \ref{app:proof}. 
\end{proof}

The proof in the LAD regression case follows:

\begin{proof}
The proof follows a similar idea to that of Theorem \ref{thm:mean}.  For any fixed $\theta$, we can always sort $\{z_i\}_{i=1}^n$ based on the error $\|y_i'-\theta^Tx_i'\|$ to satisfy the condition.
For simplicity, we denote $\tilde{\theta}=(\theta,-1)$ and $z=(x,y)$. 
By the strong duality result in Proposition \ref{prop:duality}, we have 
\begin{align*}
& {\color{white} =} \min \limits_{\QQ \in \mc R(\PP'_n)} \EE_\QQ[\|Y-\theta^TX\|]\\
& =\max_{\lambda \ge 0} \E_{\PP'_n}\left[\min_{\Delta \in \R^{d+1}}\|\tilde\theta^TZ'+\tilde\theta^T\Delta\|
+\lambda\|\Delta\|^r\right]-\lambda\delta\\
&  \mathop{=}\limits^{(a)} \max_{\lambda \ge 0} \E_{\PP'_n}\left[\min_{\|\Delta\|\leq \frac{\|\tilde\theta^TZ'\|}{\|\tilde\theta\|_*}}\|\tilde\theta^TZ'\|-\|\tilde\theta\|_*\|\Delta\|
+\lambda\|\Delta\|^r\right]-\lambda\delta\\
& \mathop{=}\limits^{(b)}\max_{\lambda \ge 0} \E_{\PP'_n}\left[\min\left\{ \|\tilde\theta^TZ'\|, \frac{\lambda \|\tilde\theta^TZ'\|^r}{\|\tilde\theta\|_*^r}\right\}\right]-\lambda\delta\\
& = \max_{\lambda \ge 0} \frac{1}{n}\sum_{i=1}^n \min\left\{\|\tilde\theta^Tz_i'\|,\frac{\lambda \|\tilde\theta^Tz_i'\|^r}{\|\tilde\theta\|_*^r}\right\}-\lambda\delta\\
& \mathop{=}\limits^{(c)}\max \left(\frac{1}{n}\sum_{i=1}^{k(\theta)-1} \|\tilde\theta^Tz_i'\|+\frac{1}{n}\left(1-\frac{n\delta'-\sum_{i={k(\theta)+1}}^n \|\tilde\theta^Tz_{i}'\|^r}{\|\tilde\theta^Tz_{k(\theta)}'\|^r}\right)\|\tilde\theta^Tz_{k(\theta)}'\|,0\right),
\end{align*}
where the third equality follows from the Holder inequality and  $\|\cdot\|_*$ is the dual norm of $\|\cdot\|$; the fourth one follows from Lemma \ref{lemma:1} and the sixth one is due to Lemma $\ref{lemma:lambda}$. 

This completes our proof. 

\end{proof}

\subsection{Linear Regression with Concave Cost Function} 
\label{app:proof-regression-concave-details}

In this subsection, we aim to illustrate an example as mentioned in Section \ref{sec:examples}. Suppose we choose a concave cost function $c$ that grows strictly slower than the loss function $f$. In such a scenario, when the budget is limited, the rectified data point $z_{\textrm{best}}$ consistently demonstrates the long haul structure, thereby ensuring automatic outlier identification properties.

\begin{example}
  Suppose that $\mc Z =\R^{d+1}$, $\ell(\theta,z) = \|y-\theta^Tx\|^2$ and the cost function is defined as $c(z,z')=\|z-z\|^{r}$ where $r = \tfrac{1}{2}$. We want to study the best rectification distribution of 
  \[\min \limits_{\QQ \in \mc R(\PP'_n)} \EE_\QQ[\|y-\theta^Tx\|^2].\]
\end{example}
As we shall see when the dual variable $\lambda$ is relatively small, the best rectification distribution will keep the long haul structure although the inner minimization problem  is no longer concave. 

For simplicity, we denote $\tilde{\theta}=(\theta,-1)$ and $z=(x,y)$. 
By the strong duality result in Proposition \ref{prop:duality}, we only have to focus on the inner minimization problem: 
\begin{align*}
    &  \min_{\Delta \in \R^{d+1}}  \left\{(\tilde\theta^T(Z'+\Delta))^2+\lambda \|\Delta\|^r\right\} \\
     = & \min_{\|\Delta\| \leq \frac{|\tilde\theta^TZ'|}{\|\tilde\theta\|_*}} \left\{(|\tilde\theta^TZ'|-\|\tilde\theta\|_*\|\Delta\|)^2 +\lambda \|\Delta\|^r\right\}.
\end{align*} 
Next, we aim at clarifying the structured information of the following one-dimensional optimization problem:  
\[
\min_{\|\Delta\| \in \left[0, \frac{|\tilde\theta^TZ'|}{\|\tilde\theta\|_*}\right]} K(\|\Delta\|):= (|\tilde\theta^TZ'|-\|\tilde\theta\|_*\|\Delta\|)^2 +\lambda \|\Delta\|^r.
\]
In general, unlike the case of absolute loss, the function $K(\cdot)$ will not be concave. However, the optimal solution of the resulting optimization problem will be active at the boundary when $\delta$ is small and the optimal value of $\lambda$ is sufficiently large. Initially, let's overlook the constraint and express the first-order optimality condition.
\begin{align*}
    & 2\|\tilde\theta\|_*(\|\tilde\theta\|_*\|\Delta\|-|\tilde\theta^TZ'|) + \frac{1}{2}\lambda\|\Delta\|^{-\tfrac{1}{2}} = 0  \\
    \Rightarrow & \|\tilde\theta\|_*^2 \|\Delta\| + \frac{\lambda }{4}\|\Delta\|^{-\tfrac{1}{2}} = |\tilde\theta^TZ'|\|\tilde\theta\|_* \\
    \Rightarrow &\|\tilde\theta\|_*^2\|\Delta\|^{\tfrac{3}{2}} +\frac{\lambda}{4} = |\tilde\theta^TZ'|\|\tilde\theta\|_*\|\Delta\|^{\tfrac{1}{2}}. 
\end{align*}
By changing the variable $\beta = \|\Delta\|^{\frac{1}{2}}$, we have 
\[
\|\tilde\theta\|_*^2\beta^3 + \frac{\lambda }{4} - |\tilde\theta^TZ'|\|\tilde\theta\|_*\beta = 0. 
\]

Now, we observe that $g(\beta) = \|\tilde\theta\|_*^2\beta^3 + \frac{\lambda}{4} - |\tilde\theta^TZ'|\|\tilde\theta\|_*\beta$, and our objective is to find the positive root of this cubic equation in one dimension.
Also, the first derivative is 
 $\nabla g(\beta) = 3\|\tilde\theta\|_*^2\beta^2-|\tilde\theta^TZ'|\|\tilde\theta\|_*$. The stationary point is $\beta^\star = \sqrt{\frac{|\tilde\theta^TZ'|}{3\|\tilde\theta\|_*}}$ and then the corresponding $\|\Delta^\star\| = \frac{|\tilde\theta^TZ'|}{3\|\tilde\theta\|_*} \in [0, \frac{|\tilde\theta^TZ'|}{\|\tilde\theta\|_*}]$. 
\[
g(\beta^\star) = \left(\sqrt{\frac{|\tilde\theta^TZ'|}{3\|\tilde\theta\|_*}}\right)^3 \|\tilde \theta\|_*^2 - |\tilde\theta^TZ'|\|\tilde\theta\|_*\sqrt{\frac{|\tilde\theta^TZ'|}{3\|\tilde\theta\|_*}} + \frac{\lambda }{4}. 
\]
\begin{enumerate}
\item When $\lambda \ge 4\left(|\tilde\theta^TZ'|\|\tilde\theta\|_*\sqrt{\frac{|\tilde\theta^TZ'|}{3\|\tilde\theta\|_*}} -  \left(\sqrt{\frac{|\tilde\theta^TZ'|}{3\|\tilde\theta\|_*}}\right)^3\right)$ (i.e., $\delta$ is sufficiently small), we have $g(\beta^\star) \ge 0$. As such, the critical point of the unconstrained optimization problem is not in the interval $[0, \frac{|\tilde\theta^TZ'|}{\|\tilde\theta\|_*}]$. In other words, we can conclude the optimal solution will be $0$ and the solution of vanilla least square is already optimal. 
\item When $g(\beta^\star)<0$, we know there are two solutions $\beta_+, \beta_-$ for $ g(\beta) = 0$ where $\beta_-<\beta_+$. We know $K(\|\Delta\|)$ will be increasing between $[0,\beta_-^2]$ and decreasing between $[\beta_-^2,\beta_+^2]$. Thus, the optimal solution will be either 0 or $\beta_+^2$ and ensures the long haul transportation structure. Different from the absolute loss, the cost function $\|z-z'\|^{1/2}$ is not powerful enough to move any points that achieve a perfect fit to the current hyperplane.
\end{enumerate}
\begin{figure}[!htbp]
    \centering
    \subfigure[Case 1]{\includegraphics[width=0.4\textwidth]{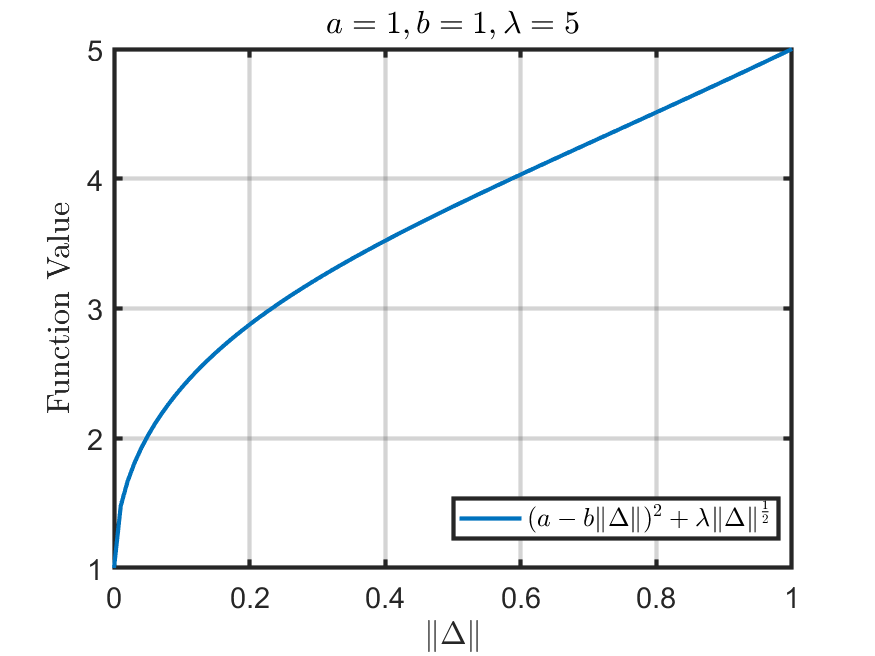}} 
    \subfigure[Case 2]{\includegraphics[width=0.4\textwidth]{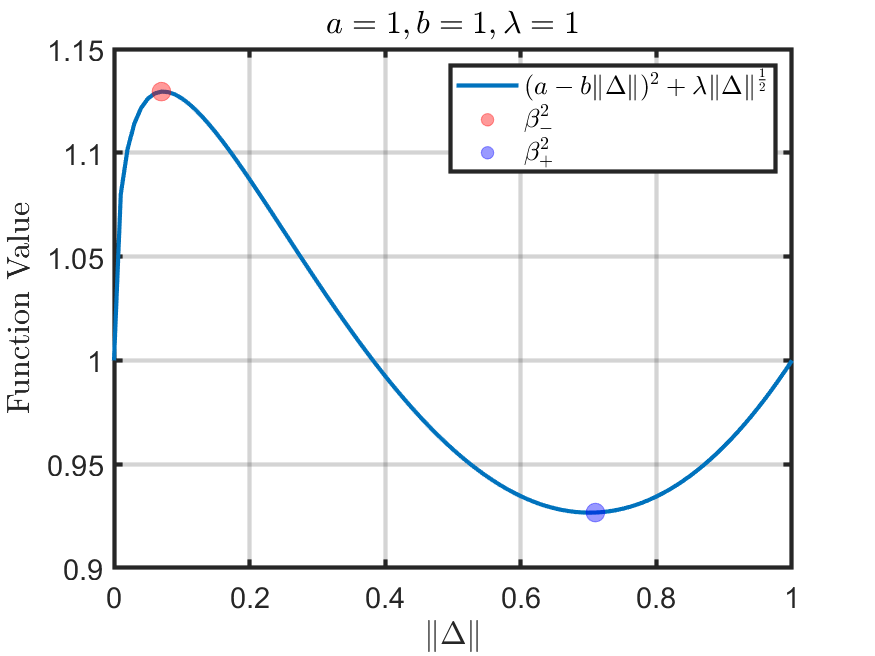}} 
    \caption{Visualization}
\end{figure}

\newpage
\section{Computational Procedure}
\label{app:algorithm}

\subsection{Alternative Estimation Procedure for Inner Problem}
\label{app:algorithm-alternative}

We make the following remark on an alternative procedure for estimating $k(\theta_t)$ and $\lambda$ in the inner problem for mean estimation and LAD regression in Algorithm \ref{alg:1}:

\begin{remark}[Alternative procedure for $k(\theta_t)$ and $\lambda^*$]
It is illuminating to note that steps 1 and 2 in Algorithm \ref{alg:1} can be alternatively replaced with steps that solve a specific linear program for the knot point $k(\theta_t)$ and optimal $\lambda^*$. Instead of applying the sorting and quick-select algorithm implied by Lemma \ref{lemma:lambda}, we can instead reformulate Equation \eqref{eq:opt_lambda} as a linear programming problem, i.e., 
\begin{equation}
    \label{eq:lp_lambda}
    \begin{array}{ccll}
         & \max\limits_{\lambda \ge0, t\in\R^n}&\frac{1}{n}\sum_{i=1}^n t_i-\lambda\delta\\
         &\st & t_i 
         \leq \|\tilde\theta^Tz_i'\|, \forall i\in [n]\\\ 
         && \|\tilde\theta\|_*^rt_i \leq \lambda \|\tilde\theta^Tz_i'\|^r, \forall i\in [n].
    \end{array}
\end{equation} 
The problem in \eqref{eq:lp_lambda} can then be solved instead of performing steps 1 and 2 of Algorithm \ref{alg:1} to find the knot point $k(\theta_t)$ and the optimal $\lambda^*$.
\end{remark}

\subsection{General Computational Procedure}
\label{app:algorithm-general}

We also propose a procedure for fitting more general regression models with our statistically robust estimator, which is based on the approach for LAD regression above. 

Although the procedure can be applied to any model which can be estimated via subgradient methods, we focus on deep learning models. We start by considering a neural network $f_\theta$ parameterized by weights $\theta$, and a loss function $\ell(\theta,z) = \ell(y,f_\theta(x))$, where $\ell$ is a gradient Lipschitz function with a constant $L$. By Proposition \ref{prop:duality}, we would like to solve the inner problem 
\begin{equation}
\max_{\lambda \ge 0} \E_{\PP'_n}\left[\min_{z\in\mc Z}\ell(\theta;z)
+\lambda c(z,Z')\right]-\lambda\delta.
\end{equation}

Although we do not prove a reformulation result for this problem, we solve the problem empirically by employing a computational procedure which is analogous to that of LAD regression. The procedure is heuristic and does not enjoy the same guarantees as Algorithm \ref{alg:1} for LAD regression, but in experiments (see Section \ref{sec:deep-learning}), we find that the performance is similarly robust to outliers and significantly outperforms benchmark models trained under empirical risk minimization. Our computational procedure follows in Algorithm \ref{alg:2}.

\begin{algorithm}[hbt!]
    \caption{Statistically Robust Optimization Procedure}\label{alg:2}
    \KwData{Observed data $\{{z_i}'\}_{i=1}^n$, initial point $\theta^{(0)}$, stepsizes $\alpha^{(t)}>0$, sampling distribution $\mathbb{P}_n'$; batch size $m \leq n$.}
    \For{$t=0,\cdots,T$}{
        1. \textbf{Sample} $m$ points $\{z_i'\}_{i=1}^m \sim \mathbb{P}_n'$. \\
        2. \textbf{Sort} the observed data $\{z_i'\}_{i=1}^m$ via the value $\ell(\theta^{(t)}, z_i')$. \\
        3. \textbf{Quick-Select} algorithm to get the knot point $k(\theta^{(t)})$ and the optimal $\lambda^\star$.\\
        4. \textbf{Subgradient step} on the detected clean data:
        \vspace{-10pt}
        \[
        \theta^{(t+1)} = \theta^{(t)} - \alpha^{(t)}\sum_{i=1}^{k(\theta^{(t)})} v^*(\theta^{(t)}, z_i')
        \vspace{-10pt}
        \]
        \,\,\, where $v^*(\theta^{(t)}, z_i') \in \partial_\theta \ell(\theta^{(t)}, z_i')$.
    }
\end{algorithm}

Algorithm \ref{alg:2} is similar to Algorithm \ref{alg:1}, but has two principal differences: (1) instead of optimizing over the entire data set, we optimize over mini-batches of size $m \leq n$ which are drawn from a data set of size $n$ via a sampling distribution $\PP_n'$; and (2) instead of taking the subgradient with respect to $\theta$ for the LAD regression problem in calculating the direction of descent, we instead take a local subgradient with respect to $\theta$ around each point $z'_{i}$, each of which is contained within  the subdifferetial $\partial_\theta \ell(\theta^{(t)}, z'_i)$. The subgradient is in practice computed by the automatic differentiation capability of a software package such as Pytorch, Tensorflow, or JAX. The budget parameter $\delta$, which controls the quantile which identifies outliers, is tuned via cross-validation.

\subsection{Optimization Details}
\label{app:algorithm-details}

Solving the optimization problems in the prior sections can be complicated, as in general, the minimization of a class of convex functions often leads to non-convex problems. 

This issue becomes even more critical when dealing with our estimation problem, as these problems may lack weak convexity or subdifferential regularity properties.
One significant challenge in such problems is the non-coincidence of different subdifferential concepts, such as Clarke, Fréchet, and limiting subdifferentials. Moreover, some calculus rules (summation and chain rules) do not always hold for these concepts. 
Thus, computing a first-order oracle, which provides necessary information for optimization algorithms, can be difficult in practice.

To concretely illustrate this challenge, we note that, intuitively, the graph of a subdifferentially regular function cannot exhibit "downward-facing cusps" \citep{li2020understanding, davis2020stochastic}. 
Here, we give an example of when the objective function \eqref{eq:mean} is not regular, which may occur in general.

\begin{example}[Irregular Case (Three Point Masses)]
\label{ref:example}
Let $\PP'_n=\frac{1}{3}\delta_{-1}+\frac{1}{3}\delta_{0}+\frac{1}{3}\delta_{1}$ and $\delta = 0.7$. Based on the reformulation result given in Theorem \ref{thm:mean}, we can plot the curve of the loss function from our proposed estimator, see Figure \ref{fig:irregular_example}.
\end{example}
\begin{figure}
\centering 
\includegraphics[width=0.5\textwidth]{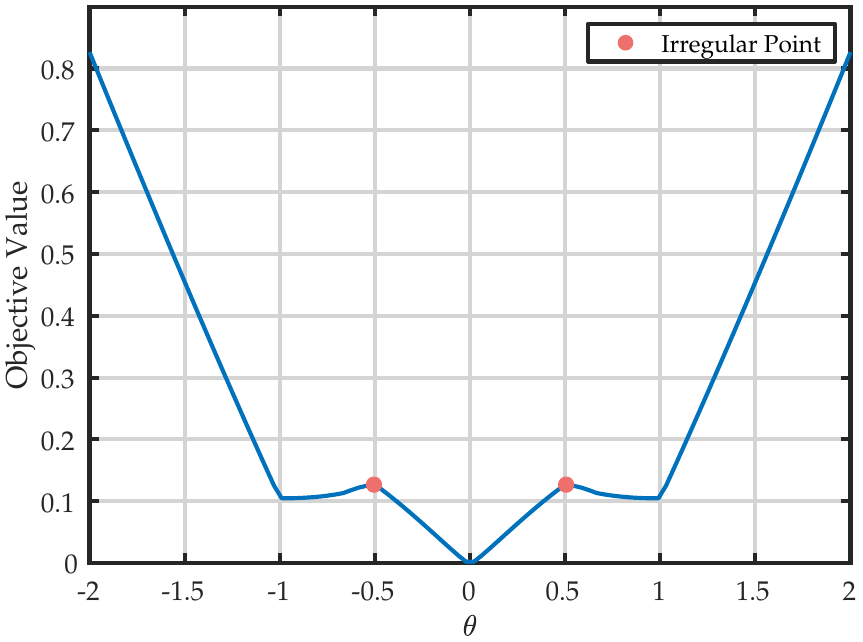}
\vspace{-10pt}
\caption{Irregular objective function.}
\label{fig:irregular_example}
\end{figure} 

From Figure \ref{fig:irregular_example}, it is evident that there are at least two regions with "downward-facing cusps" (marked as irregularities in the figures), which could potentially pose computational challenges.
Similar challenges arise in the training of deep neural networks, as demonstrated in Figure 4 of \citep{li2020understanding}, particularly for two-layer neural networks. However, the "subgradient" method remains effective for addressing these challenges empirically. In our empirical investigation, we have observed a similar phenomenon for our estimator.

\newpage
\section{Mean Estimation and LAD Regression Experiments}
\label{app:experiments-overview}

In this section, we outline the main results of the mean estimation and LAD regression experiments. Additional details are deferred to the following appendix section (Appendix \ref{app:experiment-details}) for the convenience of readers. 

\subsection{Mean Estimation}
\label{app:experiments-mean-estimation}

We generate the corrupted data by combining two Gaussian distributions: $(1-\textrm{corruption level}) \times \mathcal{N}(0,2)$ and $(\textrm{corruption level}) \times \mathcal{N}(25,2)$. All the results presented in Table \ref{tab:mean} are the mean estimators and the confidence intervals (i.e. twice the standard deviation) are averaged over 100 random trials for different corruption levels. We compare our estimator with several widely used baseline methods. Interestingly, we observe that even when setting hyperparameters as constants across all corruption levels, our estimator consistently outperforms the others. Notably, the trimmed mean, which incorporates the ground truth corruption level, performs even worse than our estimator. Our estimator's outperformance is notable, as in this table we report results specifically using $\delta=0.5$, which is a non-optimal choice for $\delta$, as we will show in the sensitivity analysis below.

\begin{table}[ht!]
\centering
\caption{
We compared our estimator with several standard mean estimation methods by evaluating the average loss on clean data points across various corruption levels. In our evaluation, we set the percent level of trimmed mean equal to the unknown ground truth corruption level. The hyperparameters for our estimator, namely $\delta = 0.5$ and $r = 0.5$, remained constant across all corrupted levels. The last row in the comparison table represents the percentage of all points rectified by our method. Error bars represent two standard deviation confidence intervals we have computed manually assuming normal errors.} 
\vspace{-2pt}
\resizebox{\columnwidth}{!}{
\begin{tabular}{@{}cccccc@{}}
\toprule
Corruption Levels & {20\%} & {30\%} & {40\%} & {45\%} & {49\%} \\
\midrule
Mean & 5.009$\,\pm\,$0.056 & 7.509$\,\pm\,$0.080 & 10.007$\,\pm\,$0.098 & 11.248$\,\pm\,$0.098 & 12.257$\,\pm\,$0.114 \\
Median & 1.680$\,\pm\,$0.114 & 1.831$\,\pm\,$0.118 & 2.286$\,\pm\,$0.192 & 2.843$\,\pm\,$0.208 & 4.203$\,\pm\,$0.344 \\
Trimmed Mean & 1.739$\,\pm\,$0.108 & 1.947$\,\pm\,$0.104 & 2.456$\,\pm\,$0.212 & 3.047$\,\pm\,$0.174 & 4.399$\,\pm\,$0.360 \\
Ours & \textbf{1.620$\,\pm\,$0.106} & \textbf{1.700$\,\pm\,$0.106} & \textbf{1.957$\,\pm\,$0.160} & \textbf{2.251$\,\pm\,$0.150} & \textbf{2.899$\,\pm\,$0.240} \\
Ours, \%Rectified & 10.03\% & 10.12\% & 10.24\% & 10.35\% & 10.55\% \\
\bottomrule
\end{tabular}}
\label{tab:mean}
\end{table}

We visualize the corrupted distribution and its rectified distribution when the corruption level is 45\% in Figure \ref{fig:mean}. In this case, 55\% of the data is drawn from the true distribution $\mathcal{N}(0,2)$ and 45\% is drawn from $\mathcal{N}(25,2)$. Figure \ref{fig:mean}(a) displays the original sample of data from the contaminated distribution. The outlier points from the contaminating distribution are clearly visible in orange. Figure \ref{fig:mean}(b) shows the rectified distribution our estimator produces for $\delta=2.5$, in which the outlier points have been moved from their original values to their rectified values, which is much closer to the true mean of the clean distribution. Our estimator thus successfully identifies the majority of outliers and relocates them towards the center (our mean estimator), providing further support for our theoretical findings in the previous section. The sensitivity analysis of mean estimation with respect to $\delta$ is displayed in Figure \ref{fig:mean}(c). In this figure, the loss on the clean data is plotted for various values of $\delta$. We see that an approximate minimum occurs at $\delta=2$, with good performance within the range $\delta \in [0.5,2.5]$. This illustrates the relative insensitivity of our estimator to different values of $\delta$ for a wide and reasonable range. This range can be easily reached by hyperparameter tuning via cross-validation, as the function describing the performance of our estimator is approximately quasiconvex. Moreover, when $\delta=0$ or when $\delta$ is set so large that it can rectify all points, our estimator gracefully degrades to the loss of the median estimator, which is another favorable property. A further sensitivity analysis with respect to $r$ is given in Appendix \ref{app:mean-estimation-details}. In this sensitivity analysis we find a similarly wide region of good performance for $r$ which improves significantly and almost uniformly on the more typical setting of $r=1$.

\begin{figure}[!htbp]
    \centering
    \subfigure[]{\includegraphics[width=0.32\textwidth]{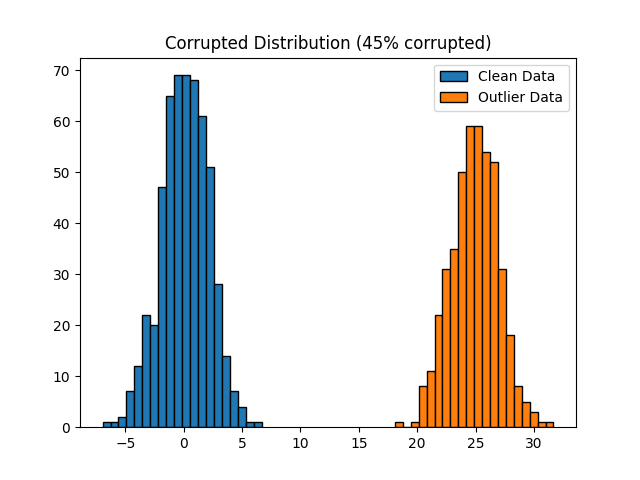}} 
    \subfigure[]{\includegraphics[width=0.32\textwidth]{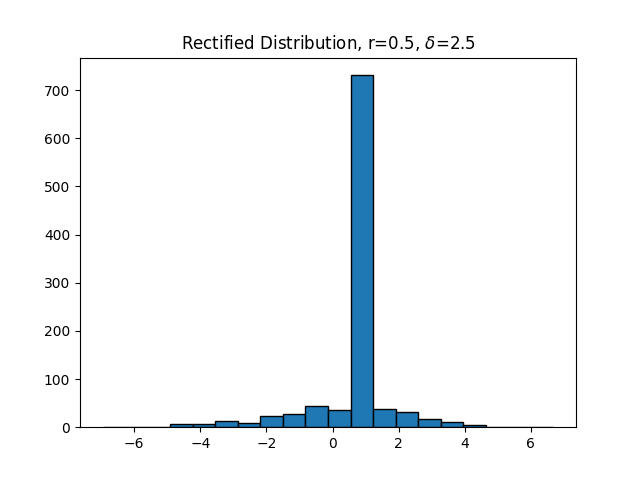}} 
    \subfigure[]{\includegraphics[width=0.32\textwidth]{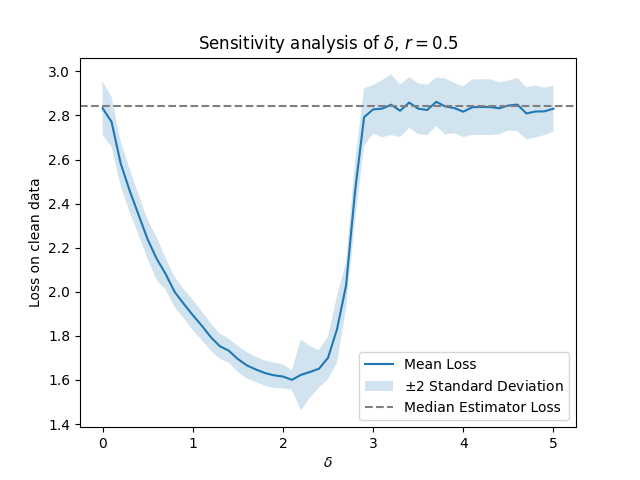}} 
    \caption{(a): Visualization of contamination model: a mixture of Gaussian $0.55\times\mathcal{N}(0,2)+0.45\times\mathcal{N}(25,2)$; (b): The rectified data generated by the proposed statistically robust estimator; (c) The sensitivity analysis of $\delta$. The visualization of (a) and (b) shows that the outlier points are rectified into the bulk of the clean distribution at $\delta=2.5$. The sensitivity analysis of $\delta$ shows that this rectification occurs for a wide range of $\delta$ which leads to losses for mean estimation which are lower than that of the median estimator of location. As $\delta \to 0$ and as $\delta$ increases, the performance of our estimator gracefully degrades to that the median estimator.}
\label{fig:mean}
\end{figure}

The evolution of the rectified distribution for mean estimation under concave and convex cost functions as $\delta$ increases is depicted in Appendices \ref{app:simulation-mean-estimation-concave} and \ref{app:simulation-mean-estimation-convex}. These depictions show that our statistically robust estimator which applies the concave transport cost function acts in a stable and expected manner as $\delta$ increases. In particular, this depiction shows that our estimator moves points in the order of their distance from the estimate of the mean; that is, the budget is ``used'' on the worst outliers first. This results in an orderly procedure of rectification which is stable and predictable across nearby values of $\delta$, which is a favorable property if $\delta$ is being set sequentially or in cross-validation. In contrast, the estimator which applies the convex transport cost function moves all points of the clean and outlier distributions toward each other, which results in poor estimates of the true mean and an improper rectified distribution.

\subsection{LAD Regression}
\label{app:experiments-lad-regression}

As we discussed in the previous result, the theoretical analysis for the LAD estimator resembles that of the mean estimation task. Empirically, we also observe similar performance for the LAD estimator and show that our estimator correctly rectifies most of outliers to the fitted hyperplane under a range of choices for $\delta$. In order to further support the effectiveness of our estimator, we provide additional visualizations: Figure \ref{fig:regression} contains a visualization of the lines of best fit produced by various estimations to the LAD regression problem, and displays the effect of different choices of $\delta$ on the line of best fit produced by our estimator. In Figure \ref{fig:regression}(a) our estimator with $\delta=1$ successfully produces a line of best fit which is closest to the uncontaminated distribution, while the other estimators (OLS, LAD, and Huber regressors) produce lines which are heavily affected by the contaminating distribution, showcasing the robustness of our estimator. Figure \ref{fig:regression}(b) and \ref{fig:regression}(c) show the rectified distribution produced by our estimator under $\delta=1$ (b) and $\delta=1.5$ (c). The suboptimal value of $\delta=1$ still rectifies many points from the contaminated distribution and produces a good line of best fit. Setting $\delta=1.5$ rectifies all of the points from the contaminated distribution and essentially recovers the true line of best fit. Importantly, this simulation shows that even improperly setting $\delta$ to the suboptimal value of $\delta=1$ produces a much better line than any of the other estimators in Figure \ref{fig:regression}(a). Appendix \ref{app:lad-regression-details} contains comparisons over different random trials and a sensitivity analysis with respect to $\delta$ and $r$, along with experimental results.
\begin{figure}[!htbp]
    \centering
    \subfigure[]{\includegraphics[width=0.32\textwidth]{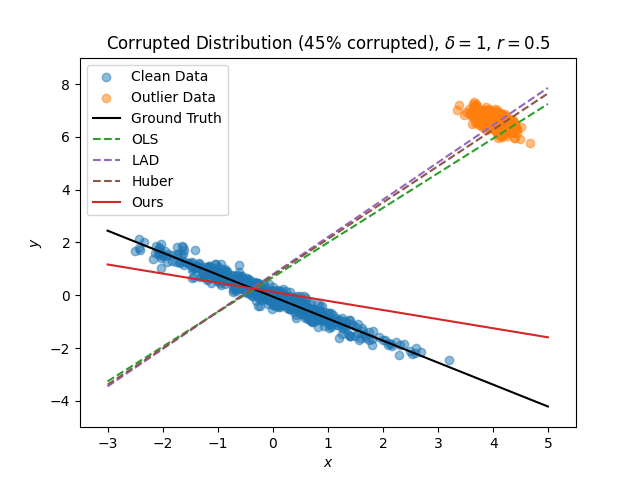}} 
    \subfigure[]{\includegraphics[width=0.32\textwidth]{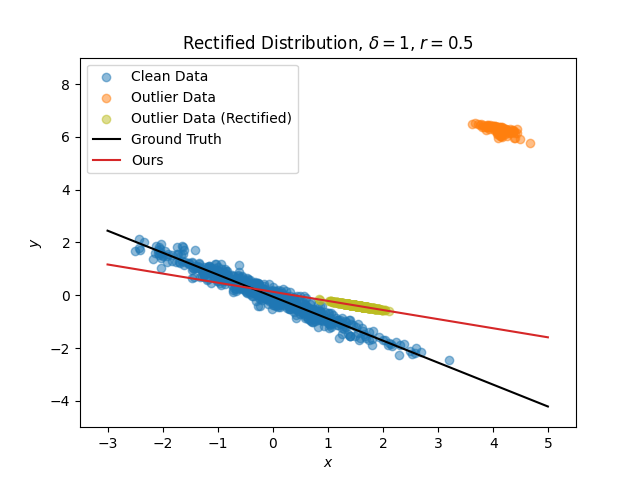}} 
    \subfigure[]{\includegraphics[width=0.32\textwidth]{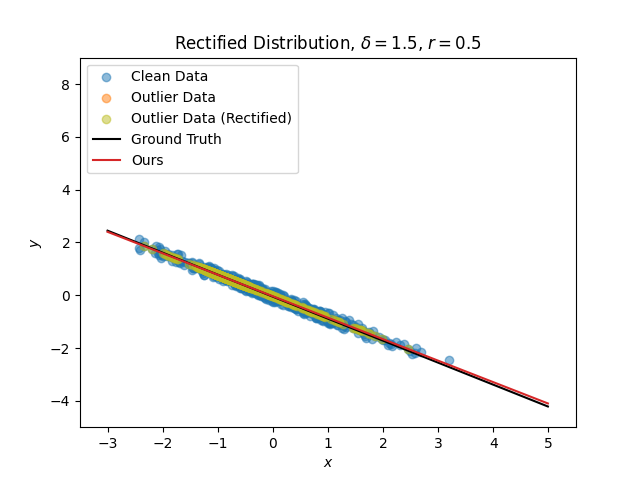}} 
    \caption{(a): Visualization of contamination model: 55\% are drawn with $x \sim \mathcal{N}(0,2)$ and $y = 0.1 - 0.8x + 0.2w$. 45\% are drawn with $x \sim \mathcal{N}(4,2)$ and $y = 10.1 - 0.8x + 0.2w$. In each case $w \sim N(0,1)$. Fitted models for different baselines are: OLS (ordinary least square), LAD (least absolute deviation regression), Huber (Huber regression with threshold parameter 1.5); (b): The rectified data generated by the proposed statistically robust estimator with a small budget $\delta=1$; (c): The rectified data with a larger $\delta = 1.5$ is able to rectify all outliers.}
\label{fig:regression}
\end{figure}

\clearpage
\newpage
\section{Additional Mean Estimation and LAD Regression Experimental Details}
\label{app:experiment-details}

\subsection{Illustrative Example Details}
\label{app:illustrative-example}

In this example, we perform linear regression using our estimator with either a concave or a convex transport cost function. The estimation is performed on a sample of points in $\R^2$ drawn from a contaminated distribution (corresponding to $\PP'$) where 45\% of the points are distributed according to an outlier distribution and 55\% of the points are distributed according to the true uncontaminated distribution (corresponding to $\PP_{\star}$). Our goal is to recover the line of best fit which corresponds to the uncontaminated distribution, which is $y = 0.1 - 0.8x $. In other words, 55\% of the points are drawn with $x \sim \mathcal{N}(0,2)$ and $y = 0.1 - 0.8x + 0.2w$, and 45\% of the points are drawn with $x \sim \mathcal{N}(4,2)$ and $y = 10.1 - 0.8x + 0.2w$. We illustrate the considerable differences which can occur when using each of these cost functions under various choices of $\delta$, the rectification budget parameter. We depict this example in Figure \ref{fig:regression-example-concave} for the concave cost and Figure \ref{fig:regression-example-convex} for the convex cost.

The behavior of the esstimator using the concave cost function is depicted in Figures \ref{fig:regression-example-concave}(a)-(c). Each of these subfigures shows the points drawn from each distribution, the rectified distribution, the true line of best fit for the uncontaminated distribution, and the line of best fit produced by our estimator, at various choices of $\delta$. In Figure \ref{fig:regression-example-concave}(a), $\delta=0$, in Figure \ref{fig:regression-example-concave}(b), $\delta=0.9$, and in Figure \ref{fig:regression-example-concave}(c), $\delta=1.5$. As can be seen from the figures, as the budget $\delta$ increases, we note that the proposed estimator using the concave cost function becomes increasingly adept at mitigating the influence of outlier data. It achieves this by rectifying the outlying points, moving them into the bulk of the uncontaminated distribution. As a result, the line of best fit (depicted in red) aligns much more closely with the true line of best fit (depicted in black). It is able to perform this successful estimation because each outlier point can be moved far with a cheap cost (a ``long haul'') due to the use of the concave transport cost function.

However, the convex cost function may result in  suboptimal estimators, which can only move every data point (clean and contaminated) a little bit. Figures \ref{fig:regression-example-convex}(a)-(c) display the behavior of the estimator using the convex cost function, which illustrates this effect. Each subfigure displays the rectified distribution as yellow-green points for a different setting of $\delta$ (i.e. (a) $\delta=0$, (b) $\delta=0.9$, and (c) $\delta=1.5$). As can be seen from the evolution of the rectified distribution, as $\delta$ increases, instead of the outlier points being moved towards the bulk of the distribution, all of the points move towards each other--even the points in the clean distribution. This is not ideal behavior, as the points from the clean distribution should stay in place. This defective behavior causes the line of best fit produced by the estimator using the convex cost function (which is depicted in red) fit to the rectified points to be severely affected by the outliers; instead of being close to the ground truth line of best fit (which is depicted in black), its slope is moved significantly upwards towards the outliers. This occurs because the convex cost function gives lower cost to smaller rectifications of the given data set, which is an inappropriate assumption for data sets containing outliers.

Appendices \ref{app:simulation-regression-concave} and \ref{app:simulation-regression-convex} depict the evolution of the rectified distribution as $\delta$ increases for the concave and convex cost functions, respectively. For our estimator which uses the concave cost, the evolution of the rectified distribution as $\delta$ increases shows that our estimator increasingly rectifies the points in order of their distance from the line of best fit in a predictable and stable way, which is a favorable property for setting $\delta$ via hyperparameter tuning. In contrast, the convex estimator consistently moves all points (clean and corrupted) toward each other as $\delta$ changes, which achieves a poor estimate of the true line of best fit for all $\delta$ considered.

\subsection{Mean Estimation Details}
\label{app:mean-estimation-details}

The experimental details for mean estimation follow. For each $(\delta, r)$ point in our experiments, we perform 100 random trials at 100 fixed seeds. In each trial, we run our optimization procedure for with a learning rate of $10^{-2}$. We stop when the number of iterations reaches 2000 or the change in the loss function between successive iterations is below a tolerance of $10^{-6}$. We initialize $\theta$ to the median of the data set.

Below in Figure \ref{fig:mean-est-sensitivity-r}, we visualize the sensitivity analysis of $r$ at $\delta=1$ when the corruption level is 45\%.  As shown, our concave transport cost function improves significantly on the linear cost function ($r=1$) and has a wide region of favorable stable performance.

\begin{figure}[!htbp]
    \centering
    \includegraphics[width=0.7\textwidth]{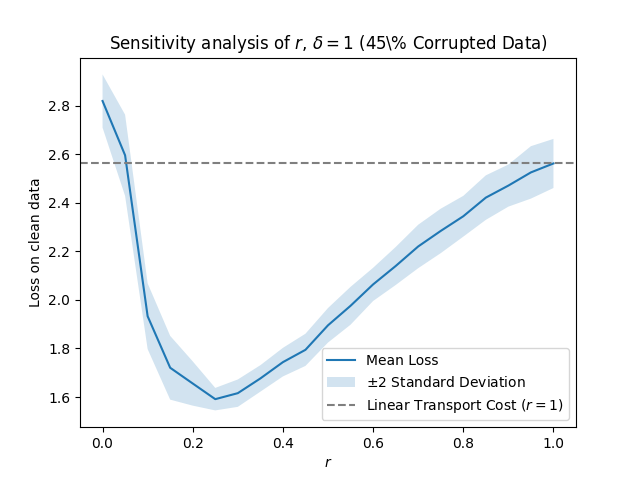}
    \caption{The sensitivity analysis of the loss on clean data with respect to $r$ of our estimator at $\delta=1$ on data with a 45\% corruption level on the mean estimation task. }
\label{fig:mean-est-sensitivity-r}
\end{figure}

\clearpage
\newpage
\subsection{Concave Cost Mean Estimation Simulation}
\label{app:simulation-mean-estimation-concave}

In Figure \ref{fig:simulation-mean-estimation-concave} below, we plot the evolution of the rectified distribution produced by our estimator under various values of $\delta$ for the \textbf{concave} cost function with $r=0.5$.

\vspace{-10pt}
\begin{figure}[!htbp]
    \centering
    \subfigure[]{\includegraphics[width=0.32\textwidth]{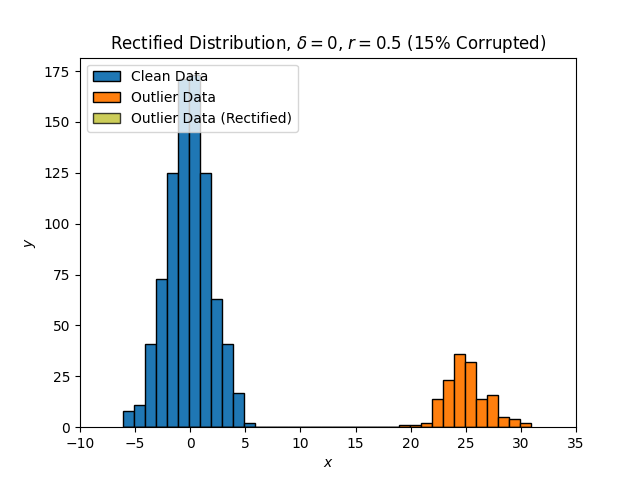}} 
    \subfigure[]{\includegraphics[width=0.32\textwidth]{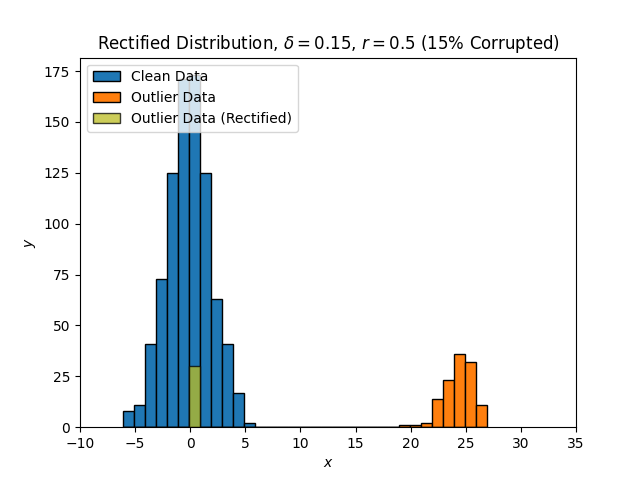}} 
    \subfigure[]{\includegraphics[width=0.32\textwidth]{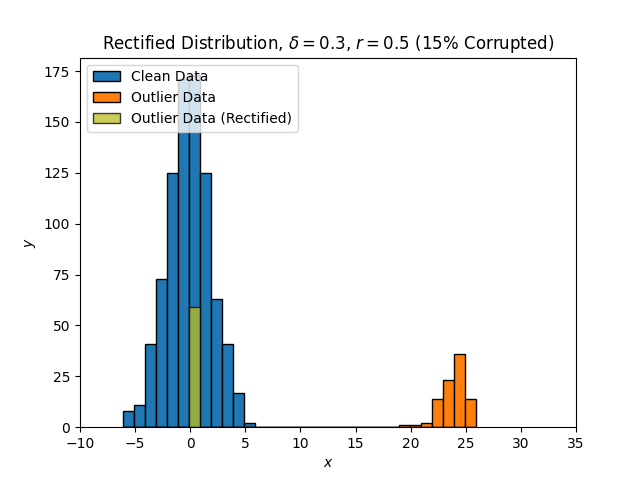}} \\ 
    \subfigure[]{\includegraphics[width=0.32\textwidth]{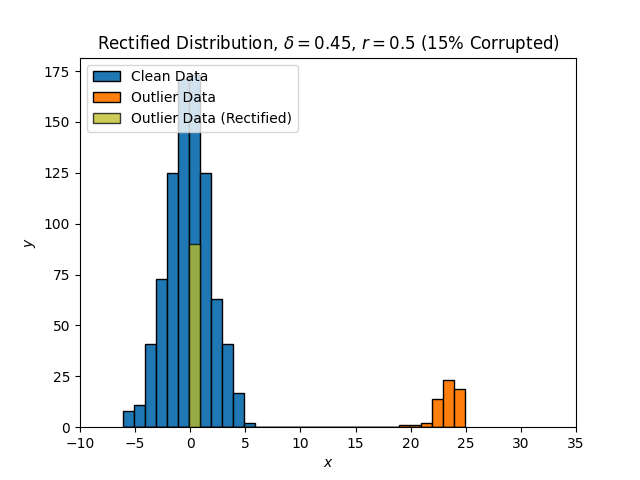}} 
    \subfigure[]{\includegraphics[width=0.32\textwidth]{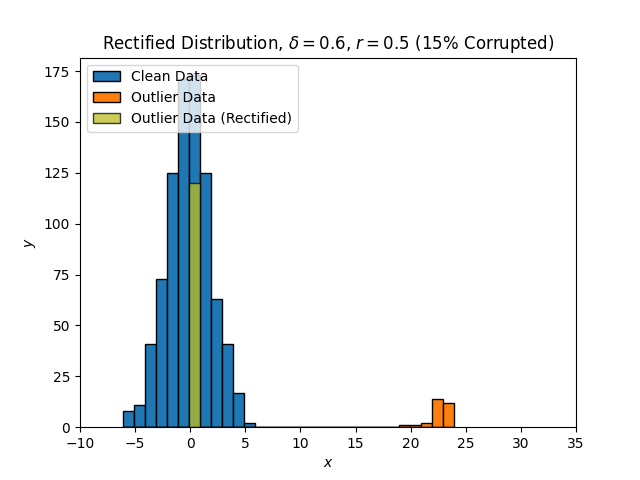}} 
    \subfigure[]{\includegraphics[width=0.32\textwidth]{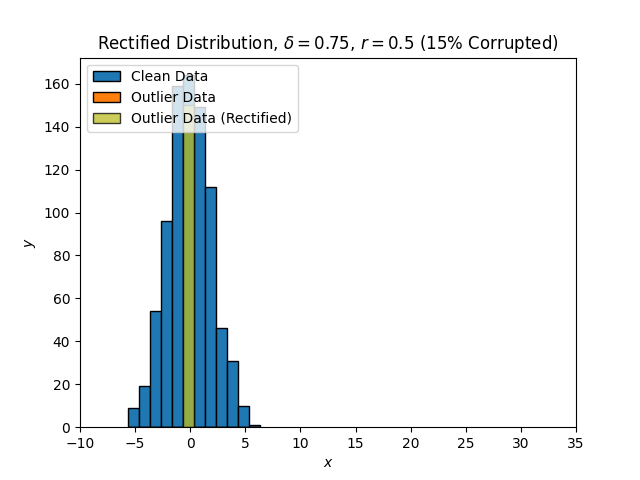}} \\
    \subfigure[]{\includegraphics[width=0.32\textwidth]{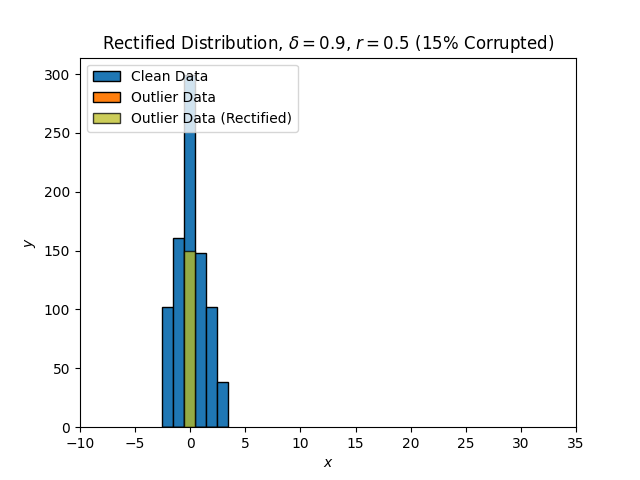}} 
    \subfigure[]{\includegraphics[width=0.32\textwidth]{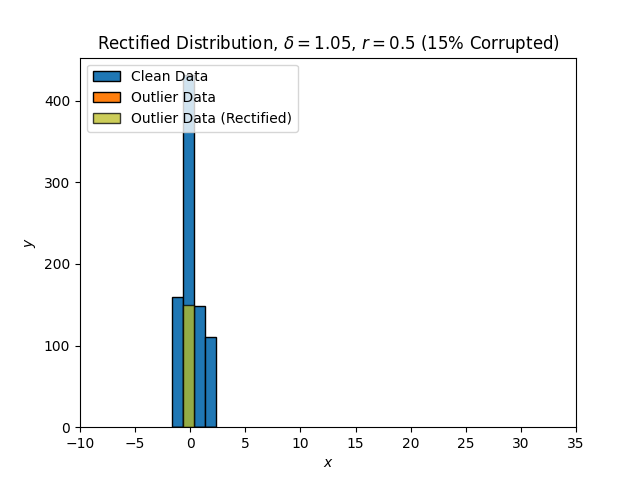}} 
    \subfigure[]{\includegraphics[width=0.32\textwidth]{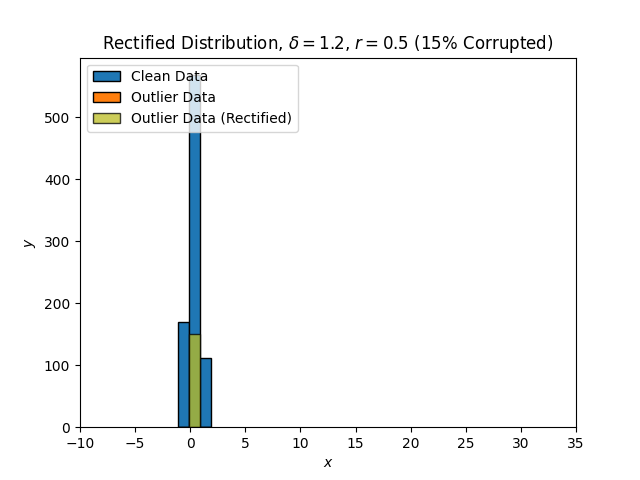}} \\
    \subfigure[]{\includegraphics[width=0.32\textwidth]{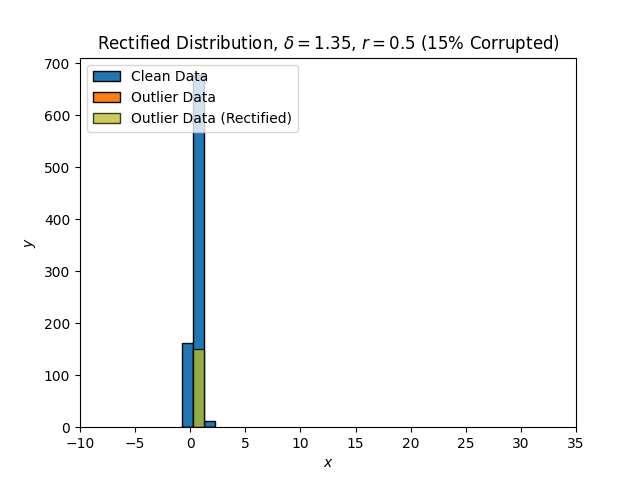}} 
    \subfigure[]{\includegraphics[width=0.32\textwidth]{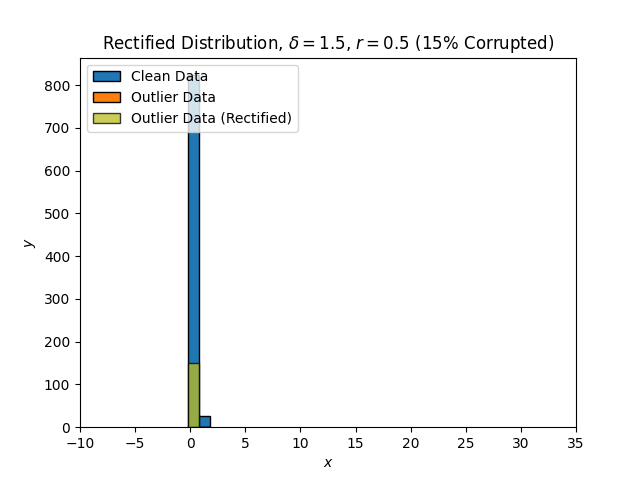}} 
    \subfigure[]{\includegraphics[width=0.32\textwidth]{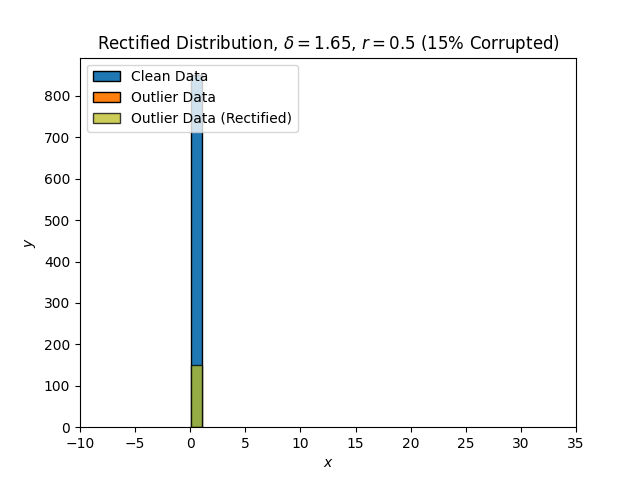}} \\
    \caption{Visualization of the evolution of the rectified distribution.}
\label{fig:simulation-mean-estimation-concave}
\end{figure}

\newpage
\subsection{Convex Cost Mean Estimation Simulation}
\label{app:simulation-mean-estimation-convex}

In the Figure \ref{fig:simulation-mean-estimation-convex} below, we plot the evolution of the rectified distribution produced by our estimator under various values of $\delta$ for the \textbf{convex} cost function with $r=2.0$.

\begin{figure}[!htbp]
    \centering
    \subfigure[]{\includegraphics[width=0.32\textwidth]{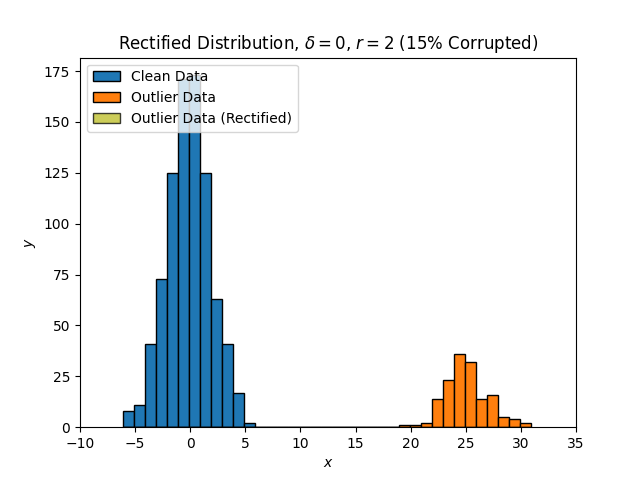}} 
    \subfigure[]{\includegraphics[width=0.32\textwidth]{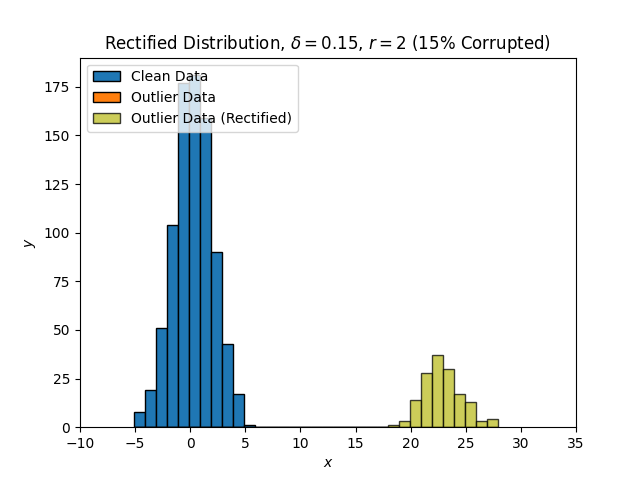}} 
    \subfigure[]{\includegraphics[width=0.32\textwidth]{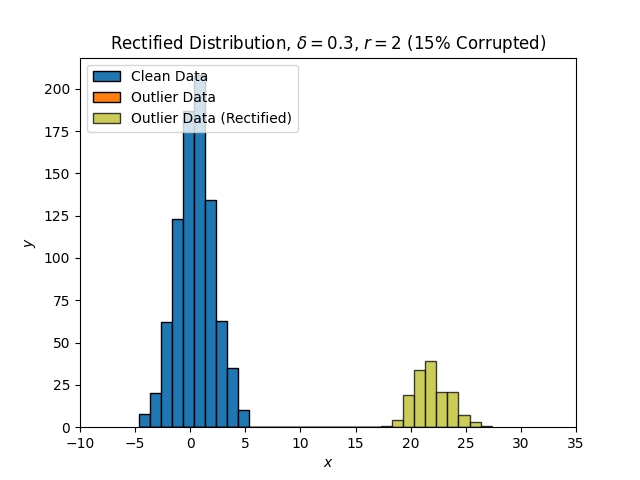}} \\ 
    \subfigure[]{\includegraphics[width=0.32\textwidth]{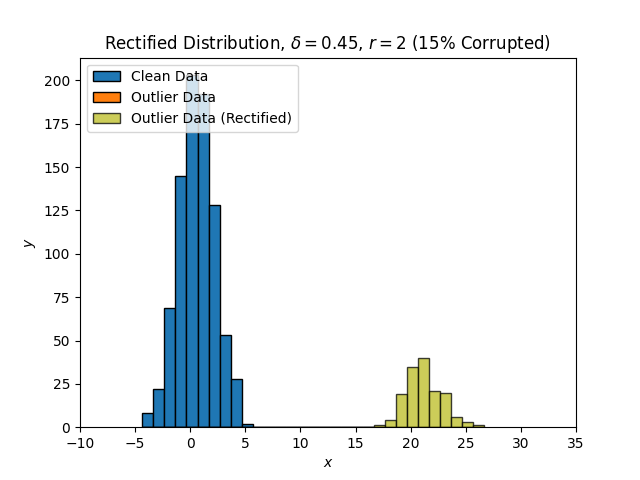}} 
    \subfigure[]{\includegraphics[width=0.32\textwidth]{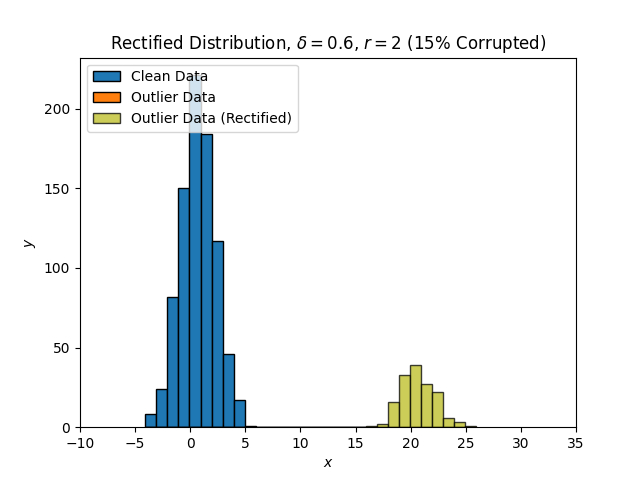}} 
    \subfigure[]{\includegraphics[width=0.32\textwidth]{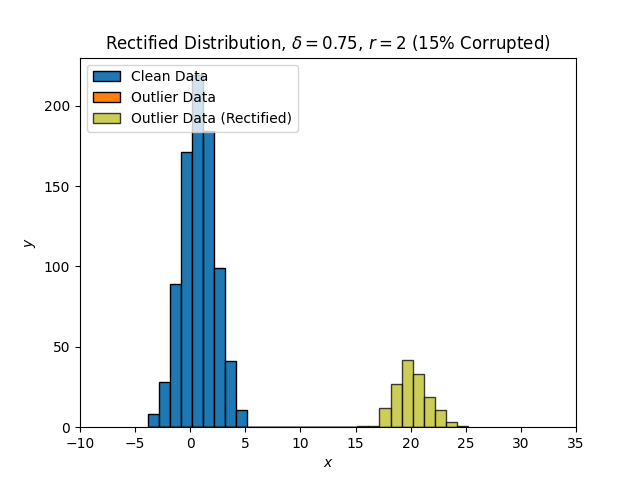}} \\
    \subfigure[]{\includegraphics[width=0.32\textwidth]{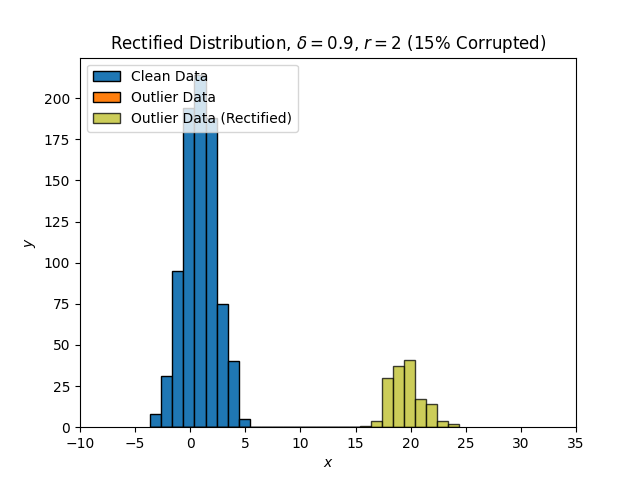}} 
    \subfigure[]{\includegraphics[width=0.32\textwidth]{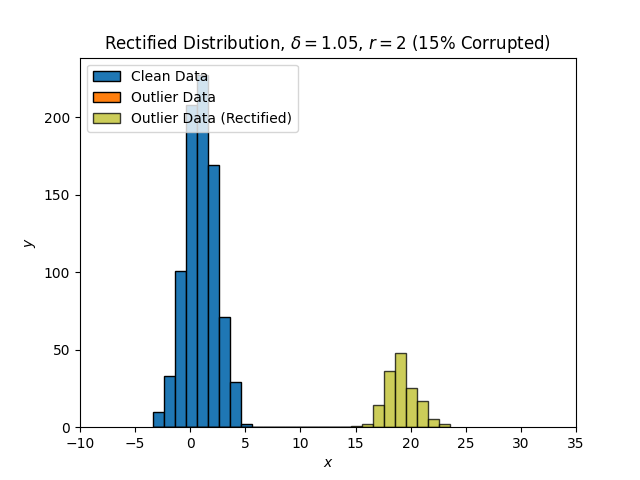}} 
    \subfigure[]{\includegraphics[width=0.32\textwidth]{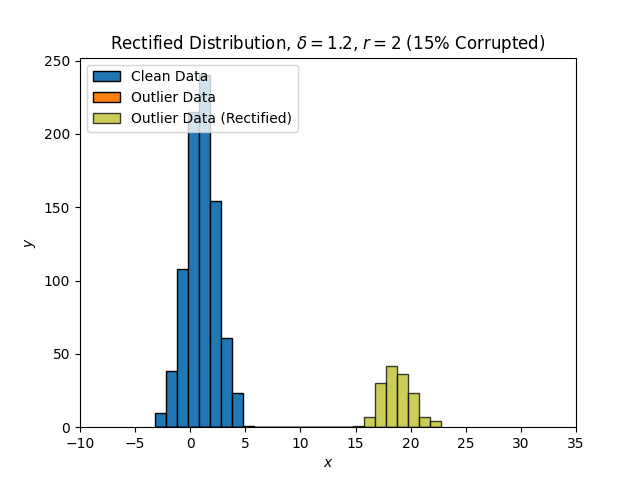}} \\
    \subfigure[]{\includegraphics[width=0.32\textwidth]{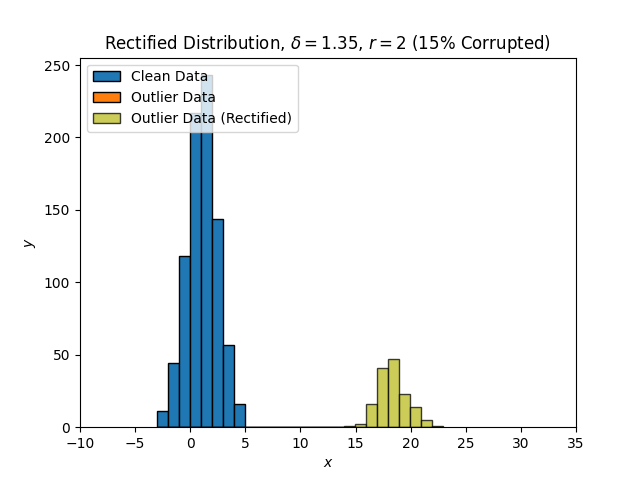}} 
    \subfigure[]{\includegraphics[width=0.32\textwidth]{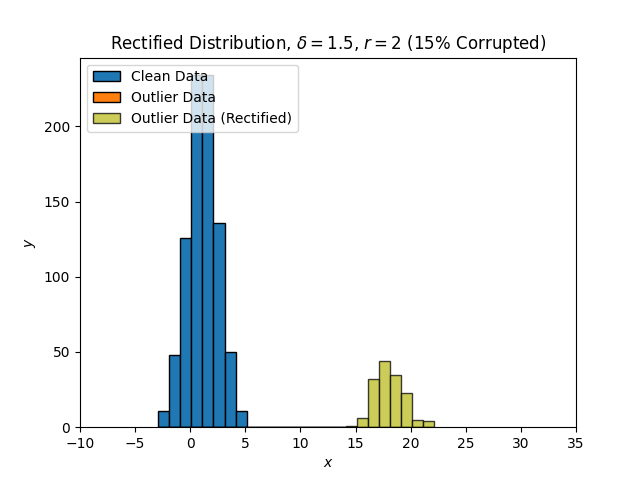}} 
    \subfigure[]{\includegraphics[width=0.32\textwidth]{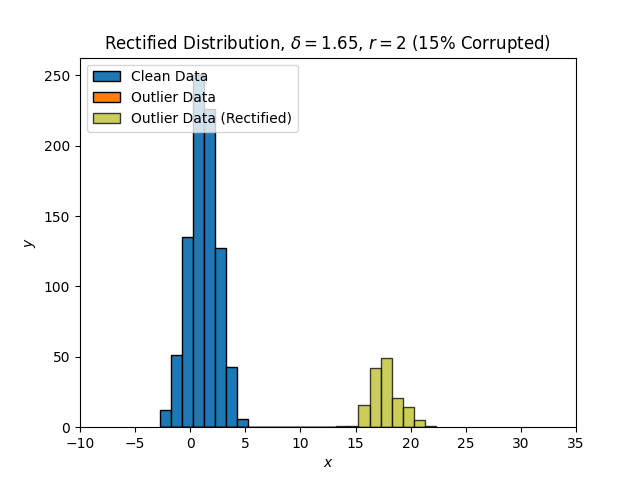}} \\
    \caption{Visualization of the evolution of the rectified distribution.}
\label{fig:simulation-mean-estimation-convex}
\end{figure}

\newpage
\subsection{Least Absolute Regression Details}
\label{app:lad-regression-details}

We provide additional results on the least absolute deviation simulations below. In Table \ref{tab:regression-results} we report the results of our experiment. We take the same experimental setting as in the regression example from the main text.

For each $(\delta, r)$ point in our experiments, we perform 100 random trials at 100 fixed seeds. In each trial, we run our optimization procedure for with a learning rate of $10^{-2}$ starting from 10 randomly initialized points and take the best loss. In each problem, we stop optimization when the number of iterations reaches 1000 or the change in the loss function between successive iterations is below a tolerance of $10^{-6}$. We initialize the slope $\beta$ according to a $N(0,1)$ distribution and the bias to zero.

The results are included in Table \ref{tab:regression-results}. As can be seen, our estimator outperforms all competing estimators, and attains significant outperformance for most corruption levels. 

\begin{table}[h!]
\centering
\caption{
We compared our estimator with several standard regression methods by evaluating the average loss on clean data points across various corruption levels. Mean performance is accompanied by a 95\% confidence interval over 100 random trials. In our evaluation, we set the threshold parameter of Huber regression to 1.5. The hyperparameters for our estimator, namely $\delta = 1.5$ and $r = 0.5$, remained constant across all corrupted levels. Error bars represent two standard deviation confidence intervals we have computed manually assuming normal errors.} 
\vspace{-2pt}
\resizebox{\columnwidth}{!}{
\begin{tabular}{@{}cccccc@{}}
\toprule
Corruption Level & {20\%} & {30\%} & {40\%} & {45\%} & {49\%} \\
\midrule
OLS & 1.569\,$\pm$\,0.041 & 1.702\,$\pm$\,0.043 & 1.773\,$\pm$\,0.049 & 1.803\,$\pm$\,0.052 & 1.822\,$\pm$\,0.054 \\
LAD & 1.808\,$\pm$\,0.131 & 1.875\,$\pm$\,0.055 & 1.892\,$\pm$\,0.059 & 1.903\,$\pm$\,0.061 & 1.908\,$\pm$\,0.062 \\
Huber & 1.642\,$\pm$\,0.042 & 1.776\,$\pm$\,0.045 & 1.842\,$\pm$\,0.053 & 1.868\,$\pm$\,0.056 & 1.882\,$\pm$\,0.059 \\
Ours & \textbf{0.657\,$\pm$\,0.638} & \textbf{0.529\,$\pm$\,0.430} & \textbf{0.680\,$\pm$\,0.467} & \textbf{0.802\,$\pm$\,0.597} & \textbf{0.866\,$\pm$\,0.619} \\
\bottomrule
\end{tabular}}
\label{tab:regression-results}
\end{table}

In Figure \ref{fig:regression-sensitivity-delta} and Figure \ref{fig:regression-sensitivity-r} below, we plot the sensitivity of our estimator across different values of $\delta$ and $r$. We see that there is a reasonable basin of good performance across both values. Each loss curve approaches the error of benchmark estimator or estimator with a traditional cost function as we let $\delta$ and $r$ tend toward values which recover these approaches.

\clearpage
\begin{figure}[!htbp]
    \centering
    \includegraphics[width=0.7\textwidth]{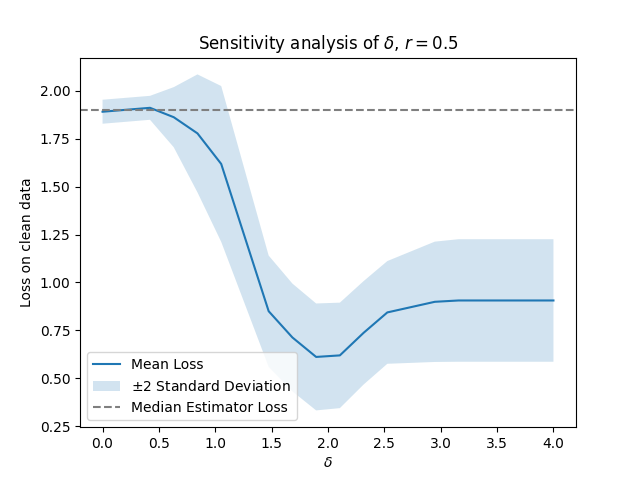}
    \caption{The sensitivity analysis of the loss on clean data with respect to $\delta$ of our estimator at $r=0.5$ on data with a 45\% corruption level on the least absolute regression task. }
    \label{fig:regression-sensitivity-delta}
\end{figure}

\begin{figure}[!htbp]
    \centering
    \includegraphics[width=0.7\textwidth]{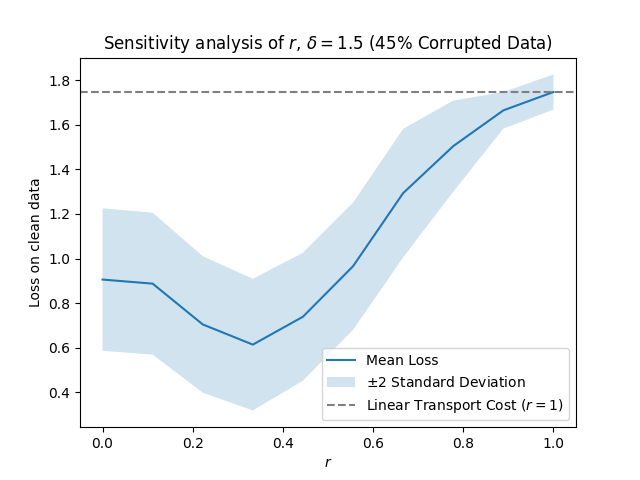}
    \caption{The sensitivity analysis of the loss on clean data with respect to $r$ of our estimator at $\delta=1.5$ on data with a 45\% corruption level on the least absolute regression task. }
    \label{fig:regression-sensitivity-r}
\end{figure}

\clearpage
\newpage
\subsection{Concave Cost Regression Simulation}
\label{app:simulation-regression-concave}

In the Figure \ref{fig:simulation-regression-concave}, we plot the evolution of the rectified distribution and line of best fit produced by our LAD regression estimator under various values of $\delta$ for the \textbf{concave} cost function with $r=0.5$.

\begin{figure}[!htbp]
    \centering
    \subfigure[]{\includegraphics[width=0.32\textwidth]{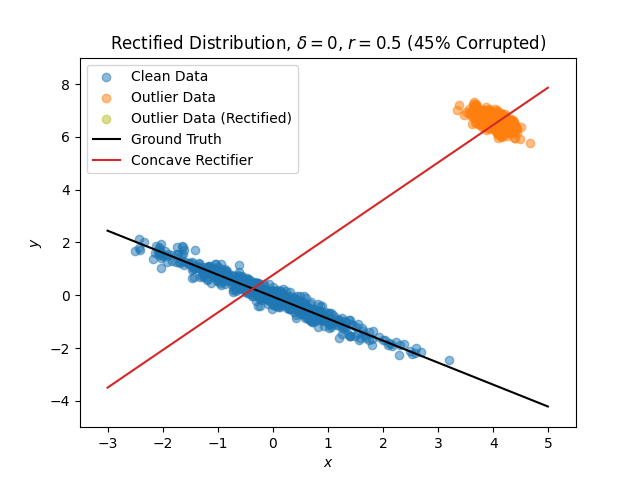}} 
    \subfigure[]{\includegraphics[width=0.32\textwidth]{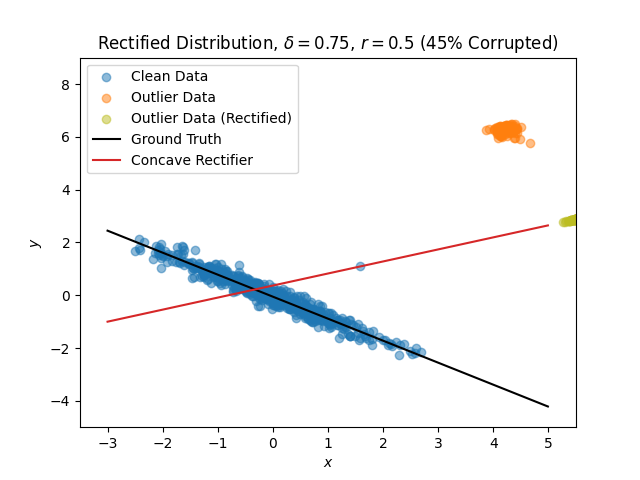}} 
    \subfigure[]{\includegraphics[width=0.32\textwidth]{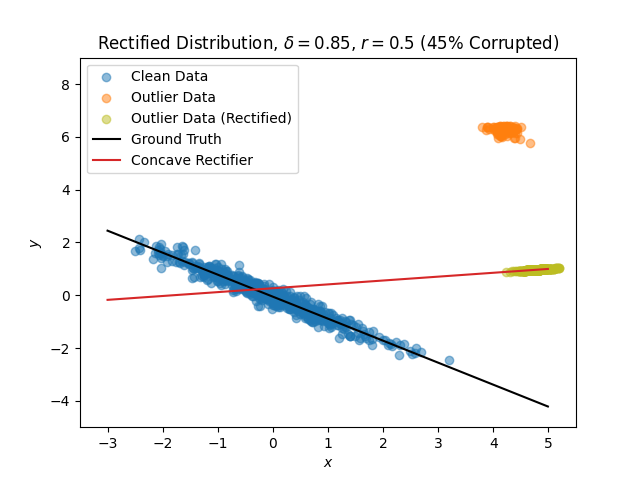}} \\ 
    \subfigure[]{\includegraphics[width=0.32\textwidth]{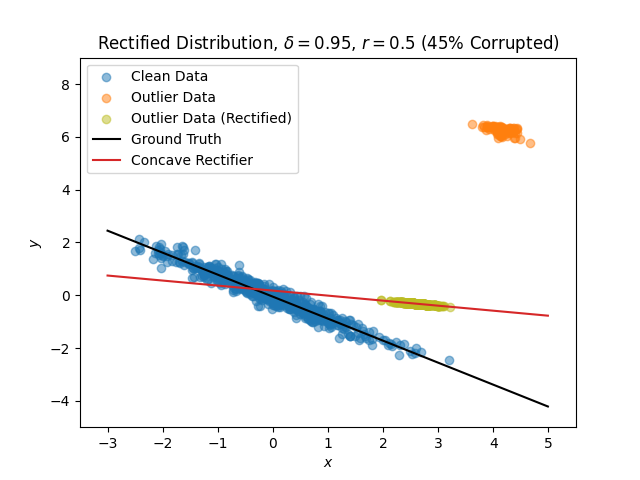}} 
    \subfigure[]{\includegraphics[width=0.32\textwidth]{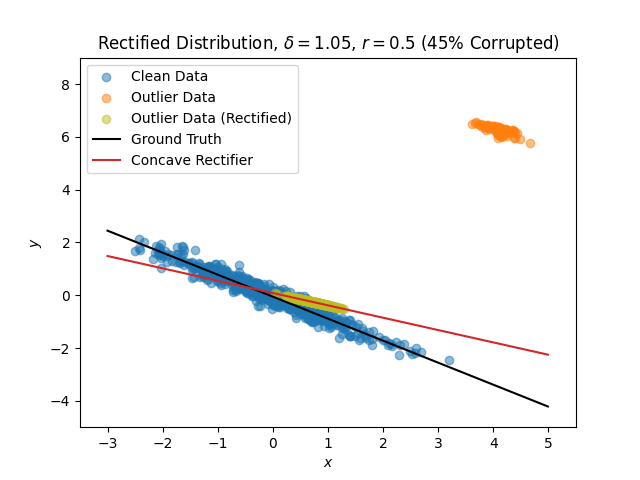}} 
    \subfigure[]{\includegraphics[width=0.32\textwidth]{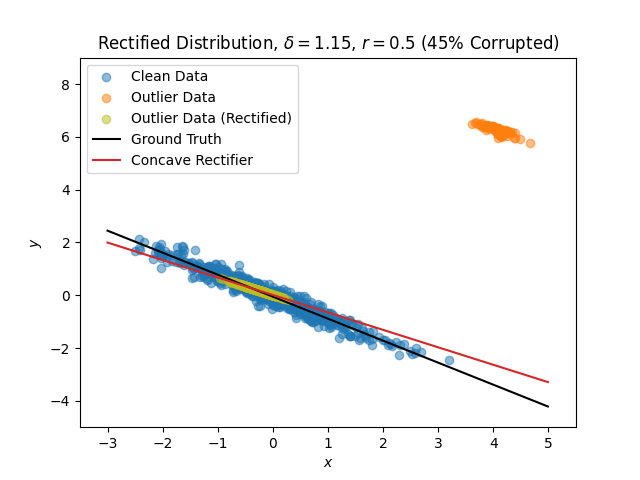}} \\
    \subfigure[]{\includegraphics[width=0.32\textwidth]{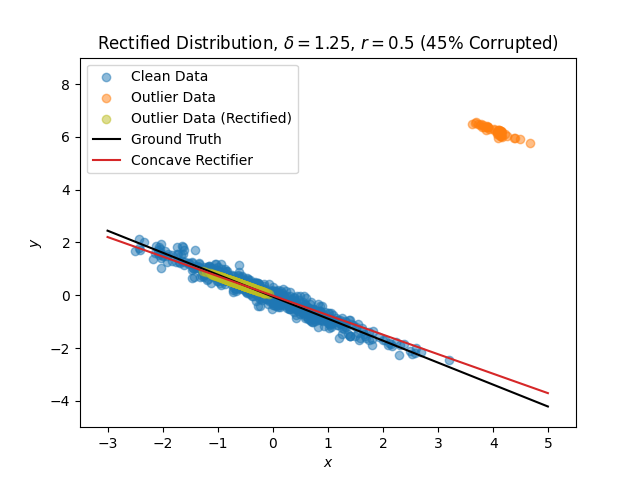}} 
    \subfigure[]{\includegraphics[width=0.32\textwidth]{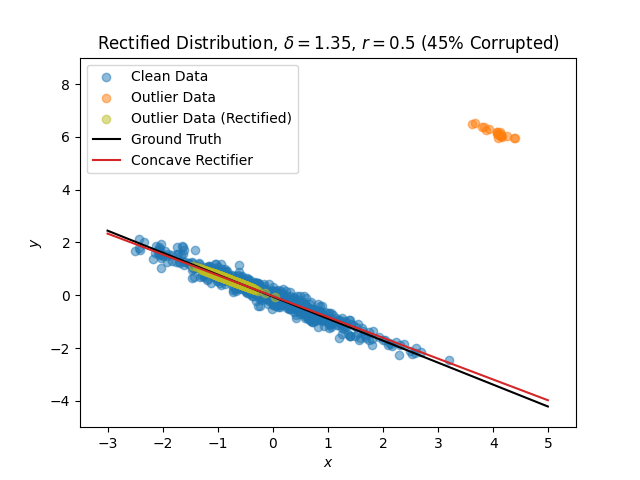}} 
    \subfigure[]{\includegraphics[width=0.32\textwidth]{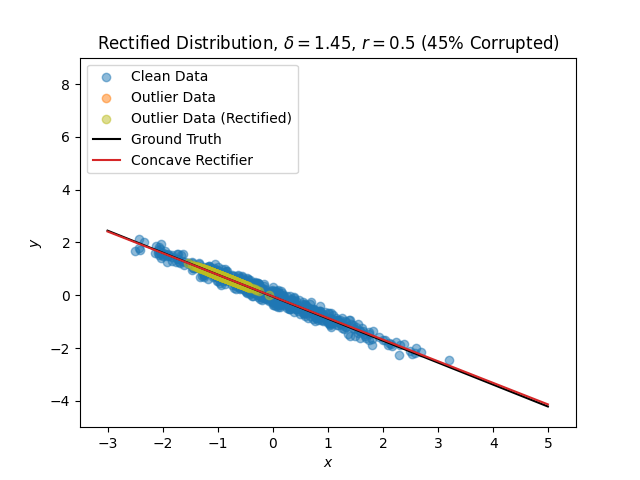}} \\
    \subfigure[]{\includegraphics[width=0.32\textwidth]{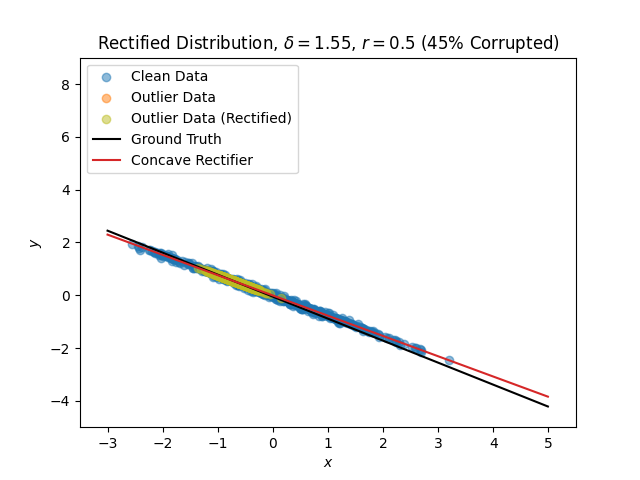}} 
    \subfigure[]{\includegraphics[width=0.32\textwidth]{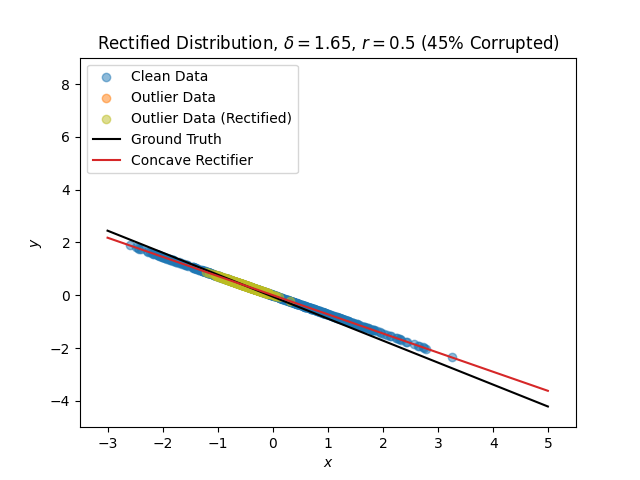}} 
    \subfigure[]{\includegraphics[width=0.32\textwidth]{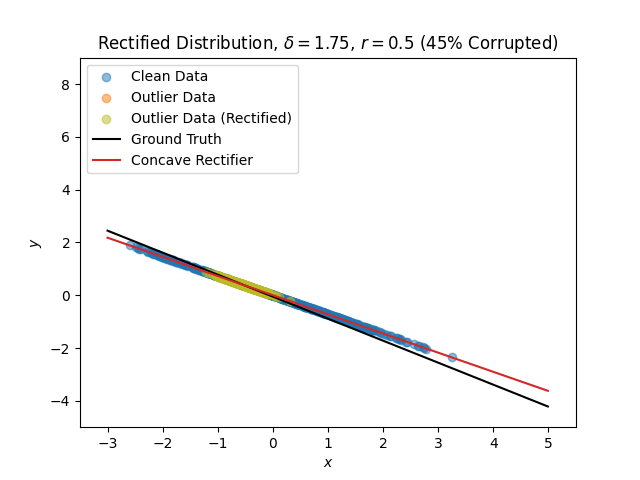}} \\
    \caption{Visualization of the evolution of the rectified distribution.}
\label{fig:simulation-regression-concave}
\end{figure}

\newpage
\subsection{Convex Cost Regression Simulation}
\label{app:simulation-regression-convex}

In the Figure \ref{fig:simulation-regression-convex}, we plot the evolution of the rectified distribution and line of best fit produced by our LAD regression estimator under various values of $\delta$ for the \textbf{convex} cost function with $r=2.0$.

\begin{figure}[!htbp]
    \centering
    \subfigure[]{\includegraphics[width=0.32\textwidth]{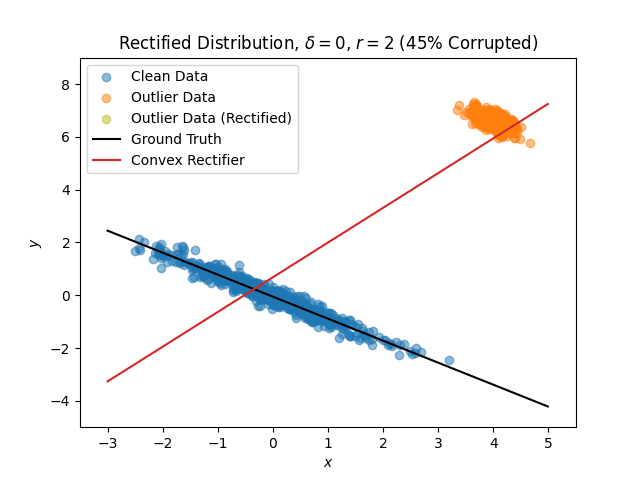}} 
    \subfigure[]{\includegraphics[width=0.32\textwidth]{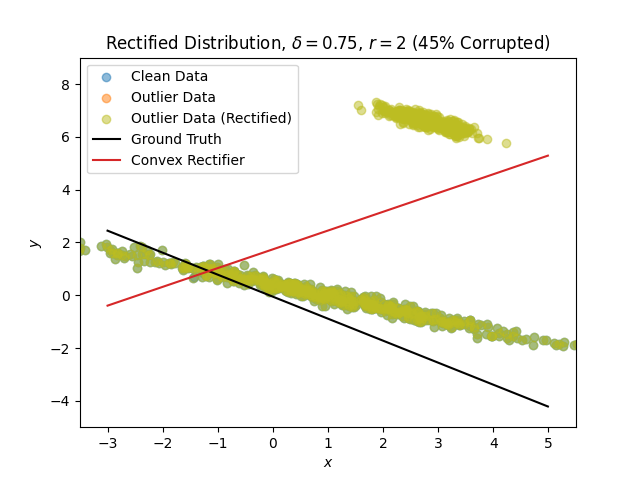}} 
    \subfigure[]{\includegraphics[width=0.32\textwidth]{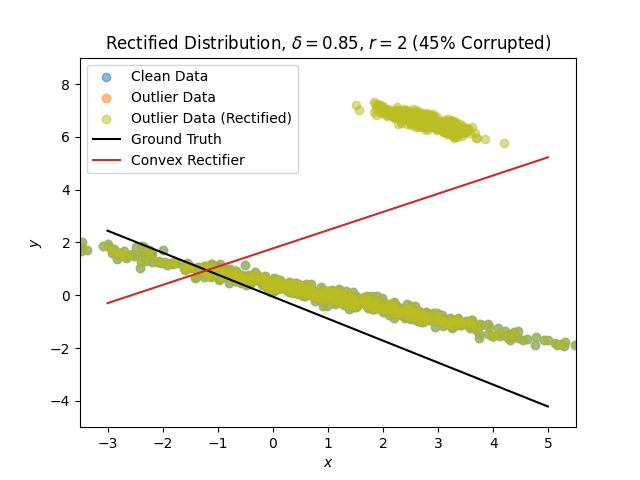}} \\ 
    \subfigure[]{\includegraphics[width=0.32\textwidth]{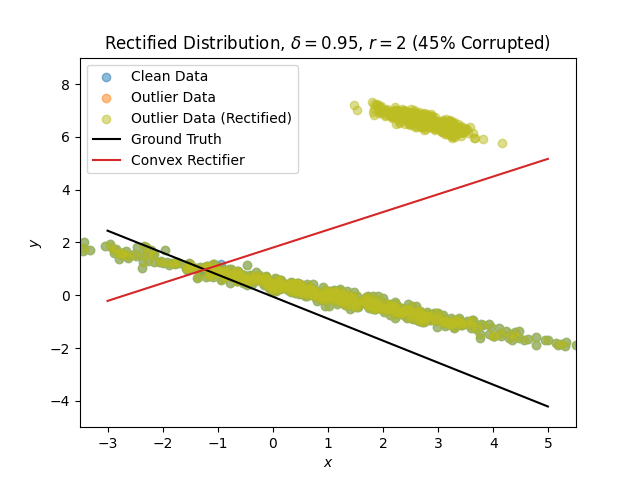}} 
    \subfigure[]{\includegraphics[width=0.32\textwidth]{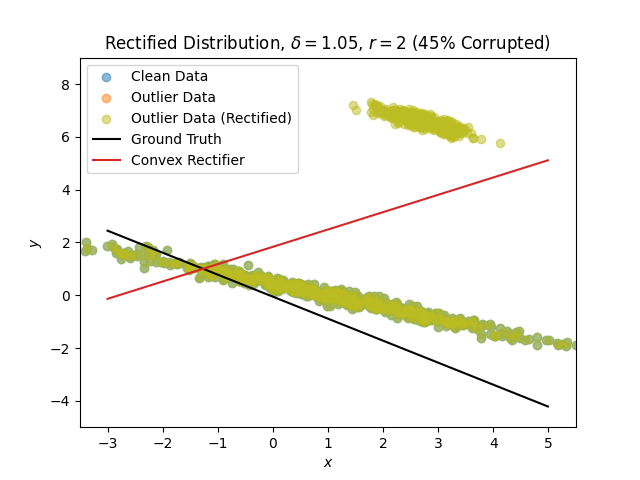}} 
    \subfigure[]{\includegraphics[width=0.32\textwidth]{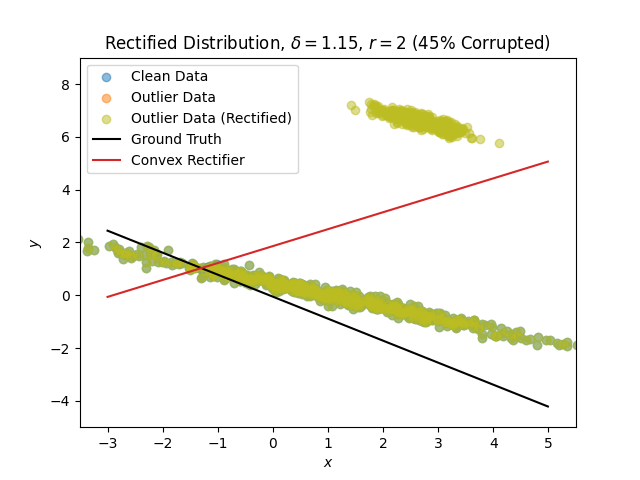}} \\
    \subfigure[]{\includegraphics[width=0.32\textwidth]{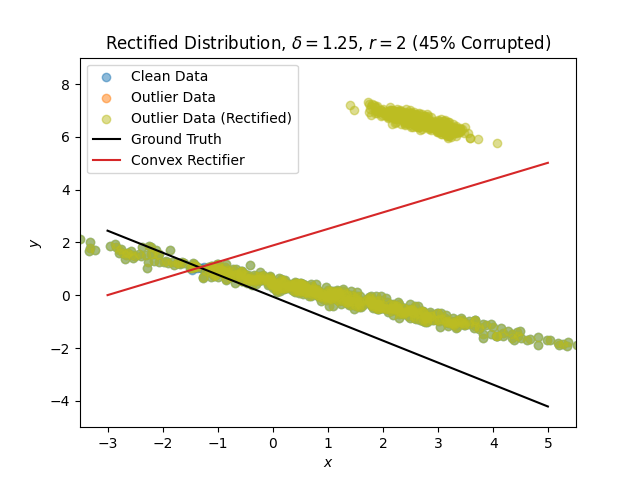}} 
    \subfigure[]{\includegraphics[width=0.32\textwidth]{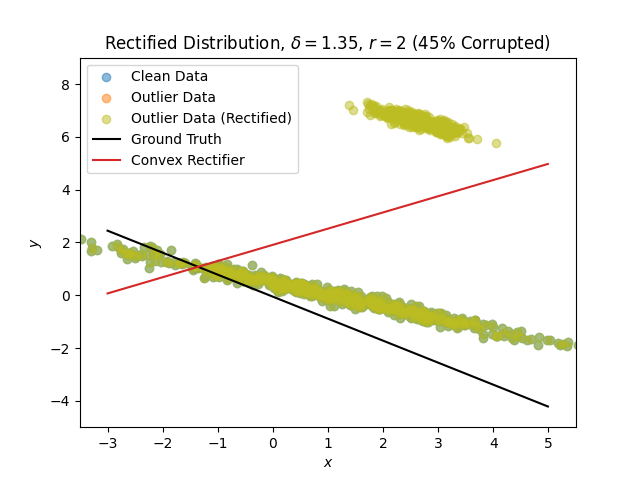}} 
    \subfigure[]{\includegraphics[width=0.32\textwidth]{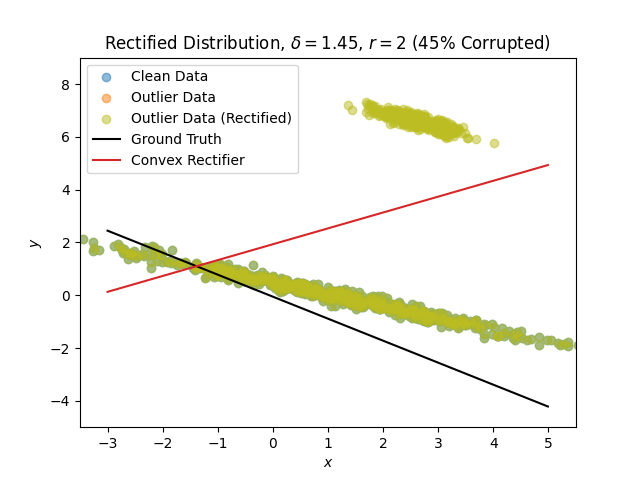}} \\
    \subfigure[]{\includegraphics[width=0.32\textwidth]{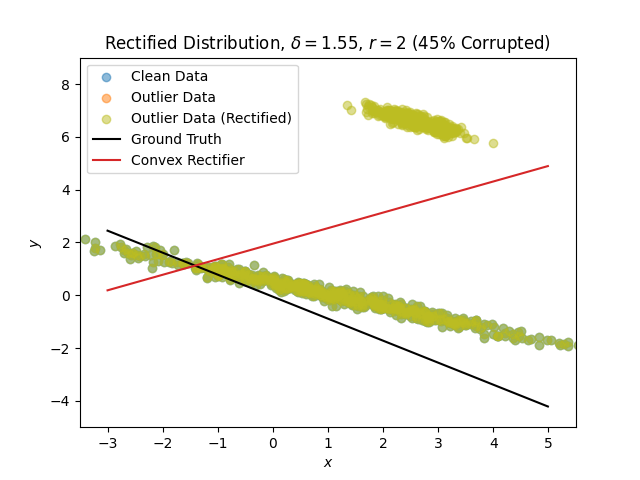}} 
    \subfigure[]{\includegraphics[width=0.32\textwidth]{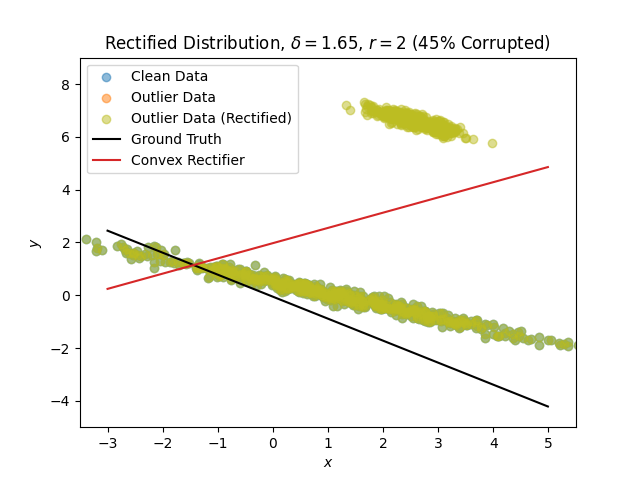}} 
    \subfigure[]{\includegraphics[width=0.32\textwidth]{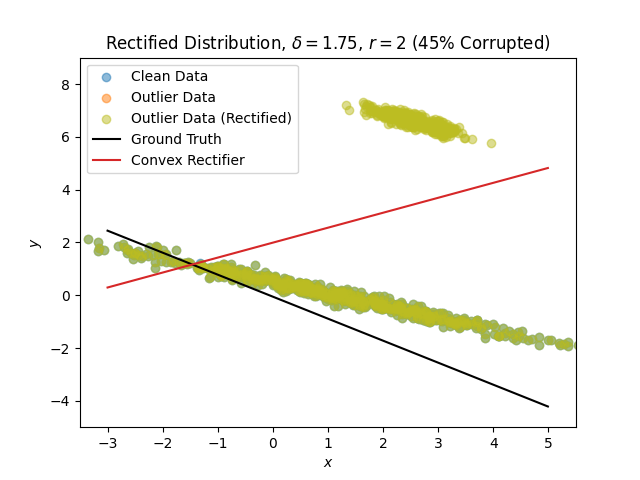}} \\
    \caption{Visualization of the evolution of the rectified distribution.}
\label{fig:simulation-regression-convex}
\end{figure}

\newpage
\section{Additional Volatility Modeling Details \& Experiments}

\subsection{Additional Volatility Surface Experiment}
\label{app:options-additional-experiments}

We have performed an additional experiment to demonstrate the practicality of our method and the selection of the parameter $\delta$ with the realistic options outlier data set of our empirical study. This experiment shows that our method leads to better out-of-sample results on complex data and does not require the availability of clean data. The problem is the estimation of the volatility surface in the presence of outliers for the option data.

Our experiment proceeds as follows. In this experiment, the data set is organized as a list of (train day, test day) tuples. Train and test days are consecutive days for the same underlying security. On the train day, we estimate our model with different methods. On the consecutive trading day (that is, the test day) we use the estimated model to obtain out-of-sample evaluation metrics for the surface (MAPE and $\dsg$). Note that this approach follows closely how volatility surfaces are used in practice by traders. This allows us to obtain the out-of-sample MAPE and $\dsg$ as a function of $\delta$. The prior Appendix section provides the details of our empirical setup.

To perform cross-validation, we split the train day into a training and validation sample, which we use to obtain estimates of MAPE and $\dsg$ as a function of $\delta$. We select the delta that optimizes MAPE and $\dsg$ and fit on the entire train day. Then, on the out-of-sample test day, we evaluate our method for the optimally selected $\delta$. Thus, to be clear, for the estimated optimal $\delta$, we estimate the surface on one day, and evaluate it completely out-of-sample on the consecutive trading day.

The results are collected in the table below. The results of our experiment show that our approach outperforms competing approaches and that cross-validation improves upon the fixed $\delta$ results reported in the main text. The below tables show the out-of-sample MAPE and $\dsg$ averaged over all out-of-sample test days.

\begin{table}[!htbp]
    \centering
    \scalebox{0.9}{
    \begin{tabular}{ccccccc} 
    \toprule
    Model MAPE & $0.5\%$ Quantile & $5\%$ Quantile & Median & Mean & $95\%$ Quantile & $99.5\%$ Quantile \\
    \midrule
    KS & 0.047 & 0.088 & 0.236 & 0.274 & 0.563 & 1.033 \\
    2SKS & 0.036 & 0.065 & 0.172 & 0.216 & 0.504 & 0.999 \\
    Ours ($\delta=10^{-2}$) & 0.037 & 0.065 & 0.170 & 0.203 & 0.440 & 0.841 \\
    Ours (CV) & 0.036 & 0.064 & 0.170 & 0.203 & 0.437 & 0.837 \\
    \midrule
    Model $\nabla\hat{S}$ & $0.5 \%$ Quantile & $5\%$ Quantile & Median & Mean & $95\%$ Quantile & $99.5\%$ Quantile \\
    \midrule
    KS & 0.323 & 1.616 & 14.042 & 19.653 & 59.534 & 105.361 \\
    2SKS & 0.036 & 0.121 & 1.759 & 6.823 & 30.020 & 76.652 \\
    Ours ($\delta=10^{-2}$) & 0.035 & 0.112 & 1.579 & 5.880 & 25.521 & 70.708 \\
    Ours (CV) & 0.033 & 0.104 & 1.248 & 4.522 & 19.645 & 55.261 \\
    \bottomrule
    \end{tabular}
    }
    \caption{Median and quantiles of the distribution of MAPE or $\nabla \hat S$ across all samples for KS, 2SKS, and our methods. Our  estimator achieves significantly lower MAPE and $\nabla \hat S$ than the benchmark KS estimator and improves upon the 2SKS estimator. Our CV procedure improves even further.}
\end{table}

\newpage
\subsection{Kernelized Regression Experiment Details}
\label{app:options-details}

\paragraph{Data set.} Our data set is derived from European-style US equity options implied volatilities from the years 2019--2021. The data come from the OptionMetrics IvyDB US database accessed through Wharton Research Data Services (WRDS). Implied volatilities created by OptionMetrics are calculated using the Black-Scholes formula using interest rates derived from ICE IBA LIBOR rates and settlement prices of CME Eurodollar futures. Option prices are set to the midpoint of the best bid and offer quoted for the option captured at 3:59 PM ET. 

We developed our options surface estimator code in partnership with a major global provider of financial data, indices, and analytical products. This provider had previously identified 1,970 outlier-containing option chains for which they found IVS estimation challenging. These outlier-containing option chains were selected from all listed US equities on days within 2019--2021, and they comprise the data set used in our experiments. Because the code is proprietary, and the data set is a proprietary selection of surfaces, we cannot publicly release them. However, the data can be found in WRDS, and we provide all implementation details in this appendix section.

\paragraph{Implementation.} We implement the KS estimator and our estimator in PyTorch. For the 2SKS method, options in an options chain are removed if their implied volatilities fall outside of the region $[q_{0.25} - 1.5 \cdot IQR, q_{0.75} + 1.5 \cdot IQR]$, where $q_{0.25}$ is the 25\% quantile of the implied volatilities in the options chain, $q_{0.75}$ is defined similarly, and IQR is the interquartile range of the implied volatilities in the options chain. The surface is then estimated upon the remaining options via the KS method. For our estimator, we estimate the surface by subgradient descent with learning rate $\alpha=10^{-1}$ and $r=0.5$, terminating when the relative change in loss reaches $10^{-5}$. We denote the test data set's option's implied volatilities $y'_i$ and the estimated surface's implied volatility for option $i$ as $\hat S(x'_i)$. We begin each iteration by setting $\delta = \ell(x,y) / 2\norm{\theta}^r$, where $\ell$ is the loss function of Theorem \ref{prop:quantile}. We initialize $\theta$ to those of the benchmark estimator. We perform 5 trials per options chain. In each trial, we randomly select a different 80\% train and 20\% test set, estimate the surface on the train set, and record the two losses (MAPE and $\nabla \hat S$) on the test set. We depict all losses gathered in this way via the histograms of Section \ref{sec:options}. Experiments are run on a server with a Xeon E5-2398 v3 processor and 756GB of RAM. 

\paragraph{Cross-validation.} The cross-validation (CV) procedure is standard. We sample 4/5 of the training surface as a CV training set and leave the remaining 1/5 as the CV validation set. Five such splits are made and used to estimate the MAPE of $\delta \in \{10^{-6}, 10^{-5}, 10^{-4}, 10^{-3}, 10^{-2}, 10^{-1}, 1, 2, 5, 10\}$. The MAPEs of each $\delta$ are averaged across the five CV runs and the $\delta$ with the lowest MAPE is then used to fit our estimator on the entire option chain.

\paragraph{Losses.} These losses are chosen for their importance in finance and option trading and valuation. MAPE is a preferable error metric for volatility surfaces, as option chains contain options with implied volatilities which can differ in orders of magnitude, and MAPE weighs equally the contribution to error of options with such volatilities. The discretized surface gradient $\dsg$ measures surface smoothness of the estimated volatility surface by computing a discrete gradient across a grid of points in the (time to exporation, strike price) plane. Smooth surfaces are important for several reasons: (1) smoother surfaces allow more stable interpolation or extrapolation of implied volatility for options with market prices which are not directly observable; (2) smoother surfaces produce smaller adjustments for option hedging positions constructed from measures of the IVS, which ultimately lowers hedging transaction costs; (3) smoother surfaces are less likely to have internal arbitrages between options, which are implied by sharp discontinuities in the IVS, putting the surface more in line with established asset pricing theory; and (4) smoother surfaces produce more stable estimates of financial institutions' derivative exposures, which are required to be consistent for financial regulatory and reporting requirements. These measures are standard in options IVS estimation. We define the option implied volatility surface $\hat{S}(x)$ as the function from \textit{option feature vectors} $x$ to \textit{estimated implied volatilities} $y$. That is, for some given $x$, we have $y:= \hat S(x)$. The accuracy measure MAPE is defined as follows for some test set of options $\{(x'_i, y'_i)\}_{i=1}^n$:
\begin{align*}
    &l_{\mathrm{MAPE}}(\hat S, \{y'_i\}_{i=1}^n) = \frac{1}{n} \sum_{i=1}^n \frac{|\hat S(x'_i) - y'_i|}{|y'_i| + c}
\end{align*}
where $c$ is a small numerical stability factor we set to $0.01$. This is a suitable choice, as $\geq$99\% of implied volatilities are greater than 0.1. 

The smoothness measure $\dsg$ is defined as follows:
\begin{align*}
    &\nabla \hat S = \sum_{i=0}^{p-1} \sum_{j=0}^{s-1} \frac{(\hat S_{\tau_{i+1},\Delta_j} - \hat S_{\tau_i,\Delta_j})^2 + (\hat S_{\tau_i,\Delta_{j+1}} - \hat S_{\tau_i,\Delta_j})^2}{2}
\end{align*}
In the expression above, $\hat S_{\tau_i,\Delta_j} := \hat S(\log \tau_i, u(\Delta_j), z)$ where $z = 1$ if $\Delta_j > 0$ and 0 otherwise. Here, each tuple ($\tau_i$, $\Delta_j$) is a member of the Cartesian product of $\{\tau_j : j \in [p]\}$ and $\{\Delta_k : k \in [s]\}$, where "$[m]$" denotes the set of natural numbers from 0 to $m$, and $p=10$ and $s=39$. By convention, OptionMetrics selects the following 11 discretization points for $\tau$, which we follow:
\begin{align*}
    \tau &\in \{10, 30, 60, 91, 122, 152, 182, 273, 365, 547, 730\}.
\end{align*}
For $\Delta$, we use the following 40 points, which are a superset of the set of discretization points which OptionMetrics selects:
\begin{align*}
    \Delta &\in \{ -1.0 + 0.05m : m \in [40] \} \setminus \{0\}.
\end{align*}

\paragraph{Example implied volatility surfaces.} In the following surface figures, we display the significant outlier rectification effect of our estimator on a sample option chain which contains outliers. These examples serves to illustrate the efficacy of the automatic outlier rectification mechanism and to provide intuition for the surface fitting problem.

\begin{figure}[h!]
    \centering
    \includegraphics[width=1\textwidth]{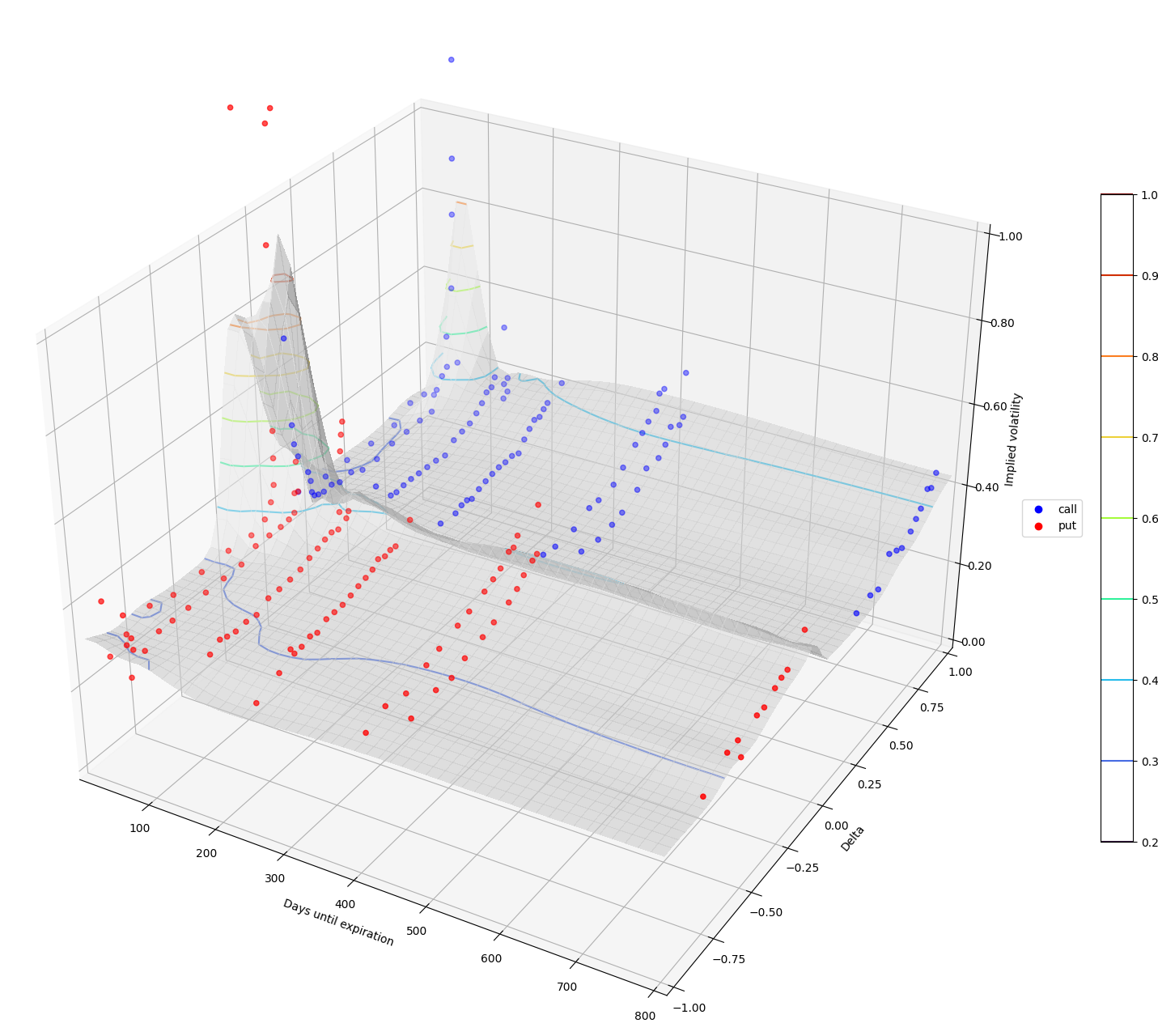}
    \label{fig:app-ks-surface}
    \caption{This plot depicts the option implied volatility surface estimated by the benchmark KS method on the options, which are depicted as blue and red dots. Blue denotes call options and red denotes put options. Outliers are present in this option chain near 0-100 days until expiration for deltas around $\Delta=-0.2$ (approximately four put option outliers and 1 call option outlier) and $\Delta=1.0$ (approximately 3 call option outliers). These outliers heavily corrupt the fitted surface, wildly distorting the values around these deltas and causing a poor fit for surrounding options which are not outliers. The surface reaches values of $~90\%$ annualized implied volatility and has a surface gradient $\dsg$ of 1.48.}
\end{figure}

\begin{figure}[h!]
    \centering
    \includegraphics[width=1\textwidth]{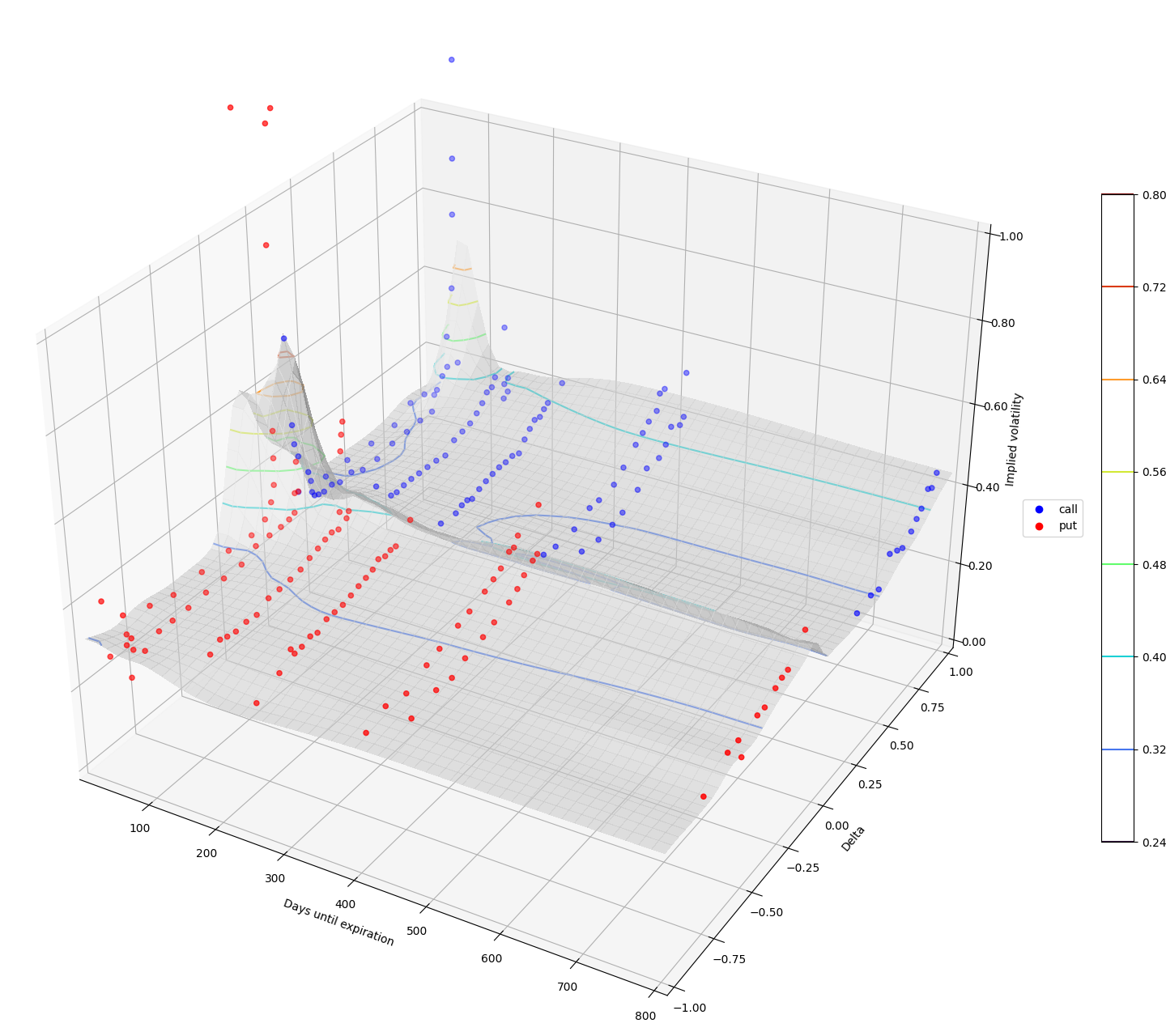}
    \label{fig:app-wks-surface}
    \caption{This plot depicts the option implied volatility surface $\hat S(x)$ estimated by our estimator with $\delta=1$ on the displayed options, which are depicted as blue and red dots. Blue denotes call options and red denotes put options. Outliers are present in this option chain near 0-100 days until expiration for deltas around -0.2 and 1.0. However, these outliers do not corrupt the fitted surface, which fits options nearby the outlying options quite well as compared to the surface fit by the benchmark method displayed in Figure \ref{fig:app-ks-surface}. Additionally, in contrast to the surface fit by the benchmark KS method, the surface fit by our estimator only reaches values of $\sim65$\% annualized implied volatility and has a surface gradient $\dsg$ of 0.69, a smoothness improvement of over 100\%. This example illustrates the capability of our estimator to rectify outliers.}
\end{figure}

\clearpage
\newpage
\subsection{Deep Learning Experiment Details}
\label{app:options-deep-learning}

The full details and description of the deep learning local volatility surface estimation experiment are given in this section.

\paragraph{Background.} IVS estimation provides an approximate surface fit for the volatilities implied by option market prices. However, market participants often require surfaces which fit as closely as possible to the market-implied volatilities.\footnote{For example, banks pricing some kinds of exotic options require such surfaces to avoid arbitrage opportunities in surfaces that their internal trading teams may exploit, as outlined in \cite{Hull_Basu_2022a}. Participants pricing options may moreover simply prefer to accept the assumption that the market prices are usually correct, and thus prefer closely fitting surfaces.} One of the most popular methods used in mathematical finance to estimate surfaces obeying such a constraint under a reasonable stochastic process model for the evolution of asset prices is the \textit{local volatility} or \textit{implied volatility function} model introduced by \cite{rubinstein1994implied}, \cite{dupire1994pricing}, and \cite{derman1996local}. The local volatility model assumptions imply that a function $\sigma(K,T)$ exists which gives an analytic expression for the volatility $\sigma$ at a given option strike price $K$ and option time to expiration $T$. This function is known as Dupire's formula after its originator. Enforcing Dupire's formula in the estimation of a surface results in a local volatility surface, which is of great use to financial market participants. However, estimating such a surface becomes significantly complicated if market data contains corruptions or unrealistic outliers which are unlikely to occur in the future.\footnote{This may occur due to low liquidity, which allows market manipulation or adversarial trading to produce unrealistic prices.} As human intervention for outlier rectification in a large number of IVS fitting tasks is unrealistic, an automatic method for rectifying outliers in models fitting local volatility surfaces is required.

\paragraph{Data.} We utilize the data set of \cite{chataigner2020deep}, which contains 17 (options chain, IVS) pairs taken from the German DAX index in August 2001. Options chains are captured from daily options market data, and IVSes are estimated for a grid of strikes and maturities via a trinomial tree calibrated by the method of \cite{crepey2002calibration}. In addition to \cite{chataigner2020deep}, this data set is also used in \cite{crepey2002calibration} and \cite{crepey2004delta}. In this data set's usage in the literature, the authors train and test using the price surface, which is given by a nonlinear transformation of the IVS. Consequently, we also use the price surface for fitting; however, we report results after transforming back to the IVS. To corrupt the data set, a percentage of the options in each options chain are selected. We refer to this percentage, the corruption level, as $\epsilon$. The prices of these options $x$ are then corrupted by replacing them with the values $10x$. 

\paragraph{Benchmark Method.} As a benchmark, we apply the state-of-the-art deep learning approach from \cite{chataigner2020deep}, the Dupire neural network. In this approach, a neural network estimates the price surface under the local volatility model assumptions encoded in Dupire's formula, and additional no-arbitrage constraints which are useful for empirical applications. These conditions are enforced using different approaches: hard, split, and soft constraints. Hard constraints enforce the conditions by separating the network into subnetworks which have separate input layers for the variables involved in each of the conditions; conditions such as convexity and non-negativity are then enforced by projection steps during optimization and the proper choice of activation functions. Hard constraints thus learn a function which explicitly satisfies the conditions. Soft constraints enforce these conditions instead by penalizing violations of them at each of the training points observed from the market. Hybrid constraints split the network into subnetworks which have separate input layers for each of the constrained variables as in hard constraints, and utilize soft constraints for enforcing the Dupire and no-arbitrage conditions. The function learned only approximately satisfies the conditions in theory, but in practice the approximation is very good (see \cite{chataigner2020deep} for details). Because these approaches estimate the price surface, we also estimate the price surface in this application; however, the local volatility surface is easily recovered from the price surface by a nonlinear transformation.

\paragraph{Robust Neural Network Estimation Problem \& Implementation.} To estimate price and volatility surfaces under the induced corruption, we estimated the benchmark approach using our statistically robust estimator. For each options chain, we tune $\delta$ via cross-validation: for each day, we sample 80\% of the training set without replacement as a cross-validation (CV) training set, and use the remaining 20\% of the training set as a CV validation set. We then train our robust neural network estimator on the CV training set with $\delta \in \{10^{-3}, 10^{-2}, ..., 10^3\}$. For each $\delta$, we then compute the MAPE on the CV validation set. We then select the $\delta$ with the lowest MAPE on CV validation set, and we use this $\delta$ to train our estimator on the entire training set. This produces our final estimated model. We base our implementation on the implementation of \cite{chataigner2020deep} (see commit e03da98 at \href{https://github.com/mChataign/DupireNN}{this url}), which is BSD 3-clause licensed. The experiments were run on a server with a Xeon E5-2398 v3 processor with 756GB of RAM.

\paragraph{Results.} To maintain comparability with the benchmark, we test and report the same metrics reported by \cite{chataigner2020deep}, namely, the RMSE and MAPE for the surface estimation. The metrics for each day and model are repeated for three trials with different random weight initialization seeds. The results are displayed in Table \ref{tab:options-stats-dl}. The first panel displays the RMSE and MAPE for each model using the different constraint techniques, averaged across all trials and corruption levels. The second panel displays the RMSE and MAPE for each corruption level averaged across all models. Our approach outperforms the baseline approach across all averages, and does so more clearly as the corruption level increases, despite the strong regularizing effect of the no-arbitrage constraints and enforcement of Dupire's formula. Our improvement is present across all models, and is especially impactful for the most accurate model utilizing soft arbitrage constraints. For this model, the out-of-sample test error in terms of RMSE and MAPE are reduced by 33\% and 34\%, respectively. This experiment displays the efficacy of our estimator in a state-of-the-art deep learning model for option price surface modeling.

\clearpage

\end{document}